\documentclass{article}

\usepackage{microtype}
\usepackage{graphicx}
\usepackage{commath}
\usepackage{booktabs}

\usepackage{hyperref}

\usepackage[accepted]{icml2025-modified}

\usepackage{amsmath}
\usepackage{amssymb}
\usepackage{mathtools}
\usepackage{amsthm}

\usepackage[capitalize,noabbrev]{cleveref}

\theoremstyle{plain}
\newtheorem{theorem}{Theorem}[section]

\newtheorem{corollary}[theorem]{Corollary}
\theoremstyle{definition}

\theoremstyle{remark}

\usepackage[textsize=tiny]{todonotes}

\icmltitlerunning{Scalable $k$-Means Clustering for Large $k$ via Seeded Approximate Nearest-Neighbor Search}

\newcommand{\defn}[1]{\textbf{#1}}
\usepackage{subcaption}

\begin{document}

\twocolumn[
\icmltitle{Scalable $k$-Means Clustering for Large $k$ via Seeded Approximate Nearest-Neighbor Search}

\icmlsetsymbol{equal}{*}

\begin{icmlauthorlist}
\icmlauthor{Jack Spalding-Jamieson}{indep}
\icmlauthor{Eliot Wong Robson}{equal,uiuc}
\icmlauthor{Da Wei Zheng}{equal,uiuc}
\end{icmlauthorlist}

\icmlaffiliation{uiuc}{Department of Computer Science; University of Illinois; 201 N. Goodwin Avenue; Urbana, IL, 61801,
USA}
\icmlaffiliation{indep}{Independent}

\icmlcorrespondingauthor{Jack Spalding-Jamieson}{jacksj@uwaterloo.ca}
\icmlcorrespondingauthor{Eliot Wong Robson}{erobson2@illinois.edu}
\icmlcorrespondingauthor{Da Wei Zheng}{dwzheng2@illinois.edu}

\icmlkeywords{approximate nearest-neighbor search, k-means, seeded search-graph, scalable, large $k$}

\vskip 0.3in
]

\printAffiliationsAndNotice{\icmlEqualContribution}

\begin{abstract}
For very large values of $k$,
we consider
methods for fast $k$-means clustering of
massive datasets with $10^7\sim10^9$ points in high-dimensions ($d\geq100$).
All current practical methods for this problem have runtimes at least $\Omega(k^2)$.
We find that initialization routines 
are not a bottleneck for this case.
Instead, it is critical to
improve the speed of
Lloyd's local-search algorithm,
particularly the step
that reassigns points to their closest center.
Attempting to improve this step naturally leads us to
leverage approximate nearest-neighbor search methods,
although this alone is not enough to be practical.
Instead, we propose a family of problems we call
\defn{Seeded Approximate Nearest-Neighbor Search},
for which we propose
\defn{Seeded Search-Graph} methods
as a solution.
\end{abstract}

\section{Introduction}
\label{sec:intro}

$k$-means clustering is a classical problem in unsupervised learning and computational geometry, with numerous applications in machine learning and data mining.
It has been thoroughly studied over the years, and it holds significant practical importance. See the recent survey of \citet{kmeans-survey-23} for a detailed discussion of this problem, its variants, and applications.

\newcommand{\Real}{\mathbb{R}}

The problem considers as input
a finite set of points $P \subset \Real^d$ in $d$-dimensional space, and a parameter $k$. The goal is to choose a set of $k$ centers $C \subset \Real^d$, $|C|=k$, minimizing the function
$\sum_{p \in P} \min_{c \in C} \norm{p - c}^2$.
In other words, we wish to find $k$ centers $C$ such that the sum of the squared distances from each input point $p \in P$ to its closest center is minimized.
Note that that the centers we choose may not be points of the original point set $P$. 
There are many alternative versions of this problem that have also been extensively studied~\cite{an2017recent}.

$k$-means clustering has long been known to be NP-hard, even for $k=2$~\cite{adhp-nphessc-09}, but a number of polynomial time approaches
are known to obtain good solutions, both in theory and in practice.
We will discuss several of them in \cref{subsec:kmeans-bg}.
The most significant is a local search algorithm
commonly
attributed to \citet{l-lsqpcm-82}
which is extremely popular in practice.
The basic version of the algorithm is as follows:

\begin{enumerate}
    \item \label{step:seed} \textbf{Initialize} A set of $k$ centers $C$ by uniform sampling from $P$.
    \item \label{step:reassign} \textbf{Assign} each point $P$ to its closest center.
    If no points change their assignment, the algorithm terminates.
    \item \label{step:avg} \textbf{Recompute} each center $C_i$ by taking the mean of points assigned to it. Return to step \ref{step:reassign}.
\end{enumerate}

We will henceforth refer to this method as \defn{Lloyd's algorithm}.
The second and third steps constitute a local search on the problem,
and are referred to as \defn{Lloyd iterations}.

Importantly, note that a single iteration of step~\ref{step:reassign} always computes
$\Theta(|P|\cdot k)$ pairwise distances,
regardless of the dataset.
For use-cases of $k$-means clustering
that require very large values of $k$,
the runtime of this algorithm has a rather impractical dependence on $k$.
To our knowledge, almost every other known competitive
approach to $k$-means clustering over large high-dimensional datasets,
both in practice and in theory,
requires at least $\Omega(k^2)$ time overall.
In particular, horizontal scaling and specialized hardware (such as GPUs)
have been the only successful approaches used to mitigate this dependence.

\subsection{Our Contribution}

In this work, we study methods to better-mitigate the dependence on $k$ in standard $k$-means that do not require specialized hardware or multiple machines.
We focus on the particularly challenging case of
large datasets ($|P| \ge 10^7$)
and moderate-to-high dimensionality ($d\geq100$).
As specific motivation, we will also
briefly discuss one possible application of this particular case
in \cref{para:out-of-core-anns}.
Additionally, in a first experiment,
we will show that the most promising path towards
designing a solution for this high-dimensional large-$k$ case
is to study improvements to Lloyd's algorithm.

For this challenging case, we will present one modified form of Lloyd's method that is quite practical, even when $k$ is almost on the same order of magnitude as $|P|$.
Moreover, our method requires no specialized hardware, requiring only
a reasonably fast CPU to out-perform GPU implementations of the best known methods at this scale.
At a high-level, our method will leverage techniques
devised for (in-memory) \defn{approximate nearest-neighbor search} (ANNS).

However, as we will see,
the most direct method for applying
ANNS techniques
does not result in practical algorithms.
Instead, we propose a more appropriate family of problems to study,
which we call \defn{seeded approximate nearest-neighbor search} (SANNS) where we have initial guesses (called seeds) for candidate nearest neighbors.
SANNS can be seen as a learning-augmented form of ANNS.
We present a framework of solutions to SANNS
that we call
\defn{seeded search-graphs}.
In particular, we present one particularly practical solution to SANNS
using this framework.
After tailoring our practical seeded search-graph approach to $k$-means,
we show that our solution is highly effective for scalable high-dimensional
$k$-means clustering with large $k$.
We call our full solution \defn{SHEESH}
(\underline{S}eeded searc\underline{H}-grap\underline{H}s for $k$-m\underline{E}ans clu\underline{S}t\underline{E}ring).
\footnote{Our code
is available at \href{https://github.com/jacketsj/mopbucket}{https://github.com/jacketsj/mopbucket}.
It includes instructions for dataset retrieval and experimental reproduction.
}

\subsection{Outline}

In \cref{sec:background} and \cref{sec:background-additional},
we discuss existing works
on $k$-means clustering,
approximate nearest-neighbor search,
and related works.
We will run several experiments in our paper,
so we present shared details of our experimental setup
in \cref{sec:exp-setup}.
In \cref{sec:init-blackbox},
we experiment with
initialization methods for large $k$ (\cref{subsec:init})
and a straightforward approach to accelerating
Lloyd's algorithm (\cref{subsec:black-box-exp}),
and conclude that these methods are not sufficient
to surpass simple forms of hardware-acceleration
in terms of practicality.
In \cref{sec:methodology},
we first present the seeded approximate nearest-neighbor search problem (SANNS),
as well as a semi-offline variant,
and present seeded search-graphs as a method
for solving SANNS.
In addition,
we discuss a highly practical seeded search-graph method
specialized for $k$-means clustering,
which consistently beats the hardware-accelerated
implementations of Lloyd's algorithm.
In \cref{sec:impl-deets},
we discuss some implementation details of this practical approach.
In \cref{sec:results-full},
we present the full set of results.
In \cref{sec:provable-sanns},
we discuss another seeded search-graph method
with some theoretical bounds.
Lastly, in \cref{sec:conclusion},
we discuss some avenues for future work.

\section{Background}
\label{sec:background}

In this section,
we discus some background on
approximate nearest-neighbor search.
We have deferred
discussion of
existing works on $k$-means clustering,
as well as existing works
applying
approximate nearest-neighbor search methods
to clustering,
to \cref{sec:background-additional}.

\subsection{Background on ANNS Methods}
\label{sec:ann_background}

In the \defn{approximate nearest-neighbor search} (ANNS) problem, the goal is to design a data structure that takes as input a set $P$ of points and a pairwise distance/similarity function on the points,
and efficiently outputs the $k'$ (approximate) nearest-neighbors to a query point $q$ in $P$ (we use $k'$ to differentiate from the $k$ in $k$-means).
For a set of $d$ dimensional points $P\subset \Real^d$, it is standard to use one of Euclidean, cosine, or inner-product functions as the distance/similarity function.
All of these are essentially equivalent for high-dimensional ANNS~\cite{bachrach2014speeding}.
This problem is also often called \defn{vector search} or \defn{vector similarity search}.
There are strong lower bounds for exact nearest-neighbor search data structures ~\cite{BorodinOR99}, as well as lower bounds in approximate settings~\cite{Liu04a}.

For our purposes, practical applications of approximate nearest-neighbor search
can be divided into two groups:
\begin{itemize}
    \item Those permitting \defn{in-memory} techniques (i.e., the entire dataset $P$ can be stored in RAM). We will refer to this as in-memory ANNS.
    \item Those requiring \defn{out-of-core} techniques (i.e., the dataset is too large to store in RAM, and is instead stored on disk or over a network). We will refer to this as out-of-core ANNS.
\end{itemize}
The distinguishing difference between these two groups is often the size of the data sets considered.
Techniques for in-memory ANNS are usually only applied to million-scale datasets,
while techniques for out-of-core ANNS are frequently applied to billion-scale datasets.
There has been significant divergence between in-memory and out-of-core techniques.

For simplicity, we can categorize some of the most relevant techniques as follows:
\begin{itemize}
    \item \defn{Quantization methods},
    such as product quantization~\cite{matsui2018pq}
    and
    vector quantization~\cite{liu2024vq},
    are (in a simplified sense)
    similar to
    dimension-reduction.
    Such methods are usually used in tandem with another technique,
    either as a method to reduce memory usage, to reduce runtime,
    or both.
    Although they are an important tool for ANNS in other cases,
    we will not need to discuss them in detail for the purposes of this work.
    \item \defn{Space-partitioning and clustering-based methods}
    constitute a very broad category of methods for ANNS.
    Theoretically-studied methods in this category
    include locality-sensitive hashing~\cite{lshsurvey}
    and RP-trees~\cite{DasguptaF08}.
    Several popular practical heuristic approaches
    include
    IVF~\cite{SivicZ03},
    IVFADC~\cite{jegou2010product,jegou2011searching},
    SPANN~\citet{chen2021spann},
    and
    ScaNN/SOAR~\cite{guo2020accelerating,sun2024soar}.
    In particular, all of these popular heuristic approaches apply some form of $k$-means clustering.
    In practice, such methods usually answer queries in two parts:
    First, a number of candidate clusters/partitions are identified
    (sometimes recursively).
    Next, they are searched, usually via another technique.
    \item \defn{Graph-search methods},
    such as HNSW~\cite{malkov2018efficient}, NSG~\cite{FuXWC19}, and NSSG~\cite{FuWC22}.
    \citet{WangXY021} give a survey of many such techniques.
    These methods answer queries using a \defn{beam search} on a (sparse) directed graph defined over the dataset (described formally in \cref{alg:beam}).
    In particular, many of the most popular methods for defining such a graph
    involve variations of a nearest-neighbor graph.
\end{itemize}

Almost all successful modern techniques for both in-memory and out-of-core ANNS
leverage some sort of quantization,
although they appear to be more critical for out-of-core ANNS (where they serve the purpose of memory-usage reduction, in addition to speed).
However, the other two categories are more clearly separated.
As a general rule, space-partitioning and clustering-based methods
are used for out-of-core ANNS,
while graph-search methods are used for in-memory ANNS.
One reason for this rule is that graph-search methods are extremely efficient in terms of number of operations,
but exhibit very poor locality,
which is important in out-of-core contexts.

\emph{We will only apply in-memory ANNS algorithms in this work}.
However,
out-of-core ANNS
is still important to discuss for another reason:
It serves as some direct motivation for accelerating $k$-means clustering,
since many out-of-core ANNS
algorithms rely on $k$-means clustering
for extremely large datasets.

\subsubsection{Out-of-core ANNS}
\label{subsubsec:out-of-core-anns}
\label{para:out-of-core-anns}
For out-of-core ANNS,
space partitioning or clustering-based methods are practically essential,
since they are the most effective tool for reducing memory usage.
Some of the most recent successful methods for out-of-core ANNS
also use a hybrid approach that additionally incorporates a graph
(see~\citet{jayaram2019diskann}, as well as some of the
2024
submissions to Big ANN Benchmarks~\cite{bigann2023}).

\paragraph{Out-of-Core ANNS Uses $k$-Means}
As stated before,
we will not be running any
out-of-core ANNS algorithms in this work,
but they do motivate
improvements to $k$-means clustering
with large $k$.
In particular,
variants of $k$-means clustering
are used in the vast majority of popular methods for space partitioning and
clustering-based approaches.
One reason for this is that $k$-means clustering over a dataset actually obtains \emph{two} things: A clustering of the dataset itself, \emph{and} a straightforward method
for assigning new points (i.e, query points) to clusters.
However, the most straightforward application of $k$-means would require many clusters,
and the query-time assignment routine would also require $k$ distance-comparisons per query, \emph{in addition} to performing a search within the chosen cluster(s).
One way
to mitigate this issue is to choose
a value of $k$ balancing the average cluster size and the total cluster count
(i.e., $k^2\approx{|P|}$).
Existing work
has either had to make this balancing tradeoff
~\cite{jegou2010product,jegou2011searching,bachrach2014speeding,baranchuk2018revisiting,JohnsonDJ21},
or apply workarounds
like hierarchical clustering (e.g., $k$-means trees)~\cite{guo2020accelerating,chen2021spann,sun2024soar}.
This presents us with a clear motivation for improving $k$-means clustering:
\begin{center}
\textit{
Mitigate the dependence on $k$ in methods for $k$-means clustering.}
\end{center}
If we could do so in a way that would also allow for efficient assignment of query points,
this would pave the way for out-of-core algorithms that
do not have to apply such workarounds or tradeoffs.
In particular, the methods we will present in \cref{subsec:black-box-exp} and \cref{sec:methodology} will do exactly this,
by using variants of \emph{in-memory ANNS}
on the cluster centers.
In this sense, one could start with an out-of-core ANNS instance (i.e., a massive dataset),
and leverage our techniques to reduce to a special variant of in-memory ANNS (i.e., a much smaller dataset).

\subsubsection{In-memory ANNS}
In contrast to out-of-core ANNS,
\emph{almost} all of the
competitive techniques for in-memory ANNS
primarily use graph-search methods.
To the best of our knowledge, the only notable exception to the dominance of graph-based techniques for in-memory ANNS
is ScaNN/SOAR~\cite{guo2020accelerating,sun2024soar},
which (at a high-level) uses a $k$-means tree and some clever quantization.
\citet{faiss-24} note that the reference implementation for both works
is thoroughly optimized
(and moreover, that the engineering optimizations are not discussed in the paper)
so it is possible this performance is more a result of careful engineering 
rather than characteristic to the algorithm.
The repository by \citet{aumuller2020ann} maintains an up-to-date benchmark of various in-memory ANNS implementations.
For an evaluation of graph-based algorithms for approximate nearest-neighbor search,
and discussions of parallelization techniques,
see ParlayANN~\cite{manohar2024parlayann}.

\subparagraph{HNSW}
One important in-memory method we will highlight now is
\defn{Hierarchical Navigable Small Worlds} (HNSW)~\cite{malkov2018efficient},
a graph-search method for in-memory ANNS
that has seen considerable industry adoption among vector search
databases and
libraries 
(e.g. \citet{Qdrant}, \citet{Milvus},
\citet{weaviate}, USearch~\cite{usearch}, and many more).
For a detailed discussion on vector similarity search databases, see the recent survey by \citet{pan2024survey}.

HNSW is an
\defn{incremental}
graph-search method,
meaning it allows for both queries
and insertions.
At a high-level,
most incremental graph-search methods (including HNSW)
maintain a sparse subgraph of
an approximate $k'$-nearest-neighbor
graph over the dataset $P$ for some value of $k'$.
Queries are performed with this structure
using beam search
over the graph.
Beam search (see \cref{alg:beam}) is sometimes called greedy search or best-first search
in this context.

\begin{algorithm}[tb]
   \caption{Beam Search}
   \label{alg:beam}
\begin{algorithmic}
   \STATE {\bfseries Input:} $P\subset\Real^d$, search-graph $G=(P,E)$, $p^*\in P$, $q\in\Real^d$, $p^*\in P$, $b\in\mathbb{Z}_{\geq1}$
   \STATE Initialize sets $C,N=\{p^*\}$ (candidates, nearest).
   \STATE Mark $p^*$ as visited.
   \REPEAT
   \STATE Extract the element $c$ from $C$ nearest to $q$.
   \IF{$|N|=b$ and $d(c,q)>d(n,q)$ for all $n\in N$}
     \STATE \textbf{break}
   \ENDIF
   \FOR{each (outgoing) neighbor $v$ of $c$ in $G$}
     \IF{$v$ is not marked as visited}
     \STATE Mark $v$ as visited
     \IF{$|N|<b$ or $d(v,q)<d(n,q)$ for some $n\in N$}
       \STATE Add $v$ to $C$ and $N$
       \STATE If $|N|>b$, remove the furthest element in $N$.
       \STATE If $|C|>b$, remove the furthest element in $C$.
     \ENDIF
     \ENDIF
     \STATE Mark $v$ as visited.
   \ENDFOR
   \UNTIL{$C$ is empty}
   \STATE {\bfseries Output:} $N$, the $b$ points in $P$ close to $q$
\end{algorithmic}
\end{algorithm}

For HNSW in particular,
these searches have an initial starting point
which is an approximate nearest-neighbor
from a random subsample of the data.
This subsample can be performed successively,
and a sparse graph can be maintained
at each level (that is, a search-graph is maintained for each level, whose initial search points are determined within the level above), allowing for the starting point of a search to be
determined recursively.
This is illustrated in \cref{fig:hnsw-demo}.

\begin{figure}[ht]
\centering
\includegraphics[width=0.15\textwidth,page=10]{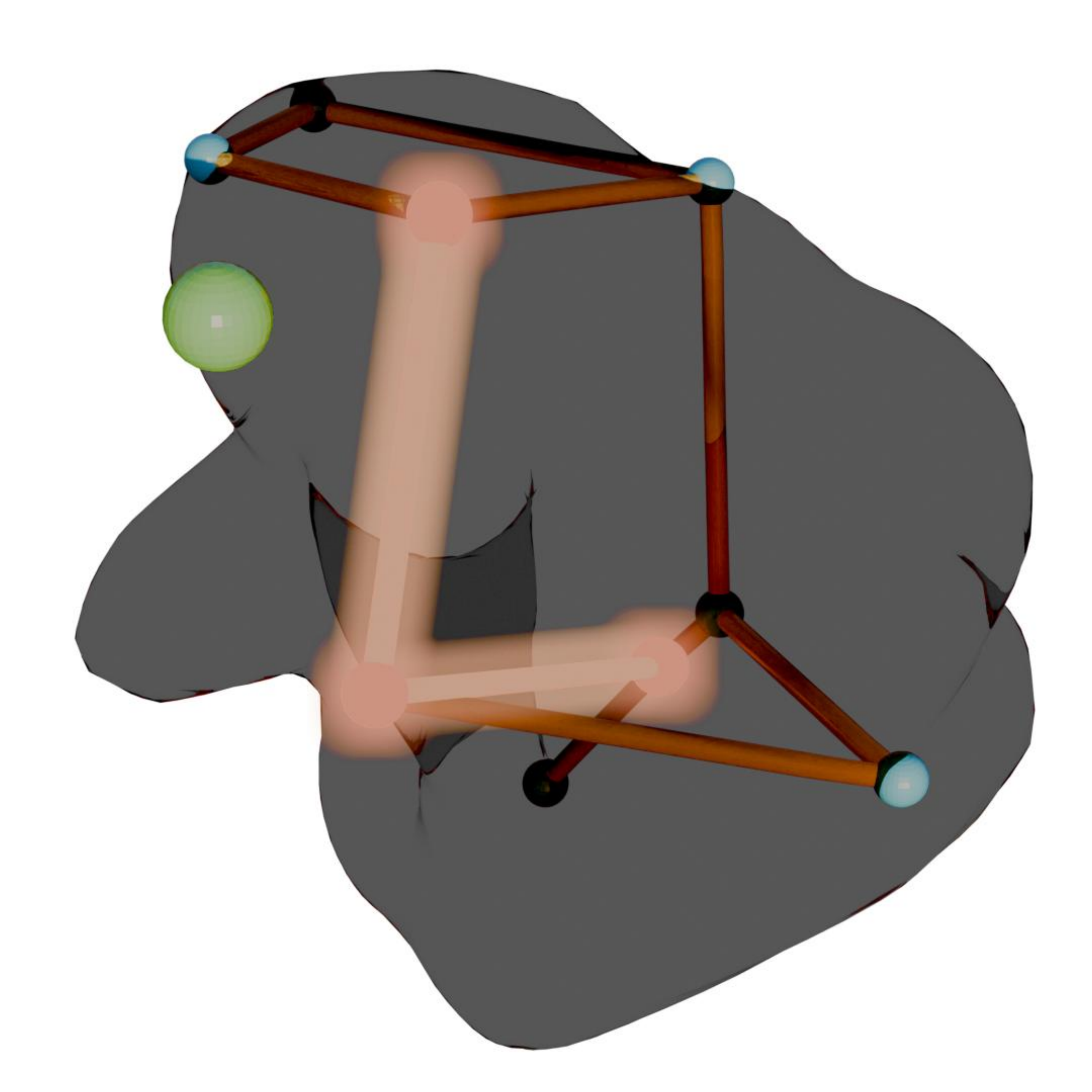}
\includegraphics[width=0.15\textwidth,page=12]{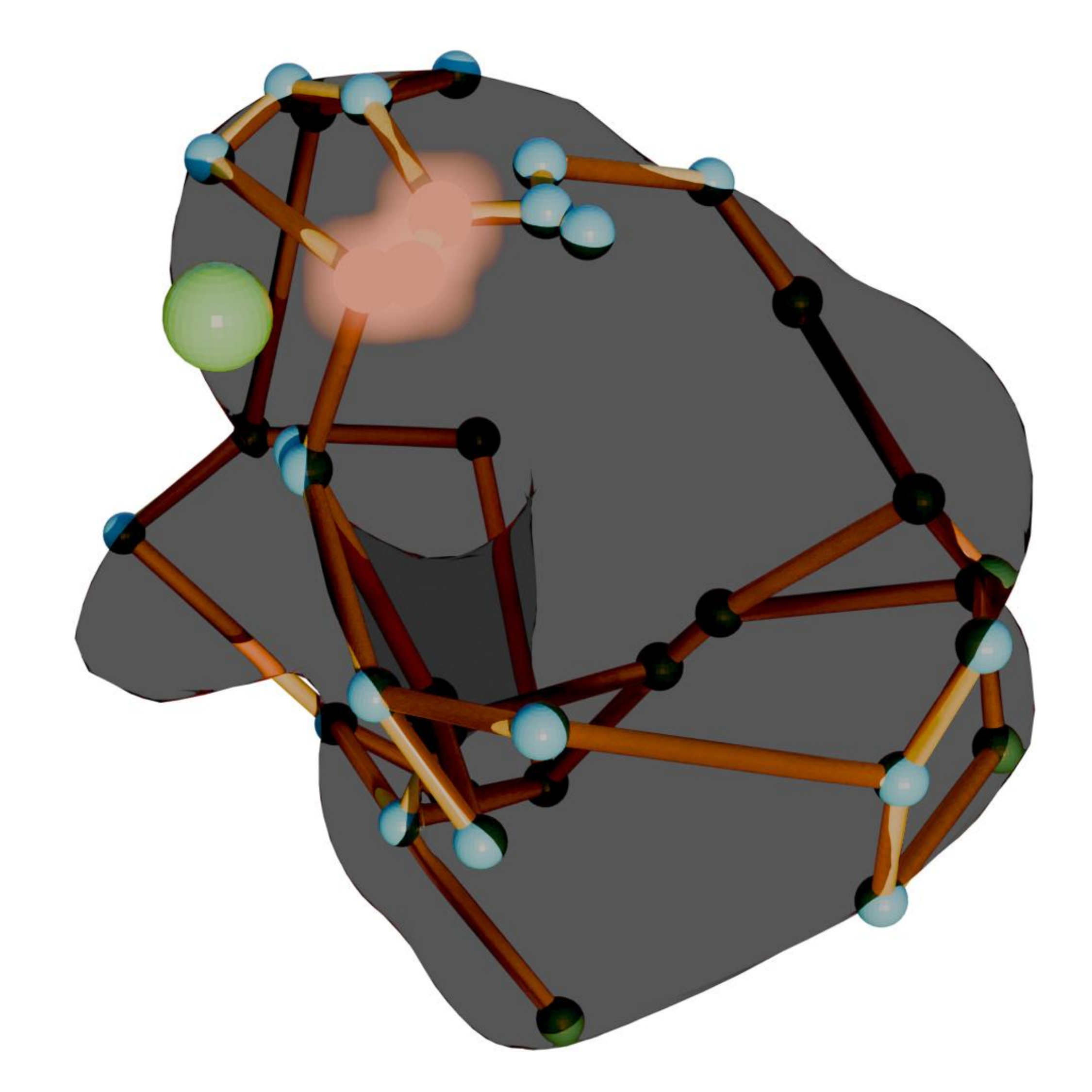}
\includegraphics[width=0.15\textwidth,page=14]{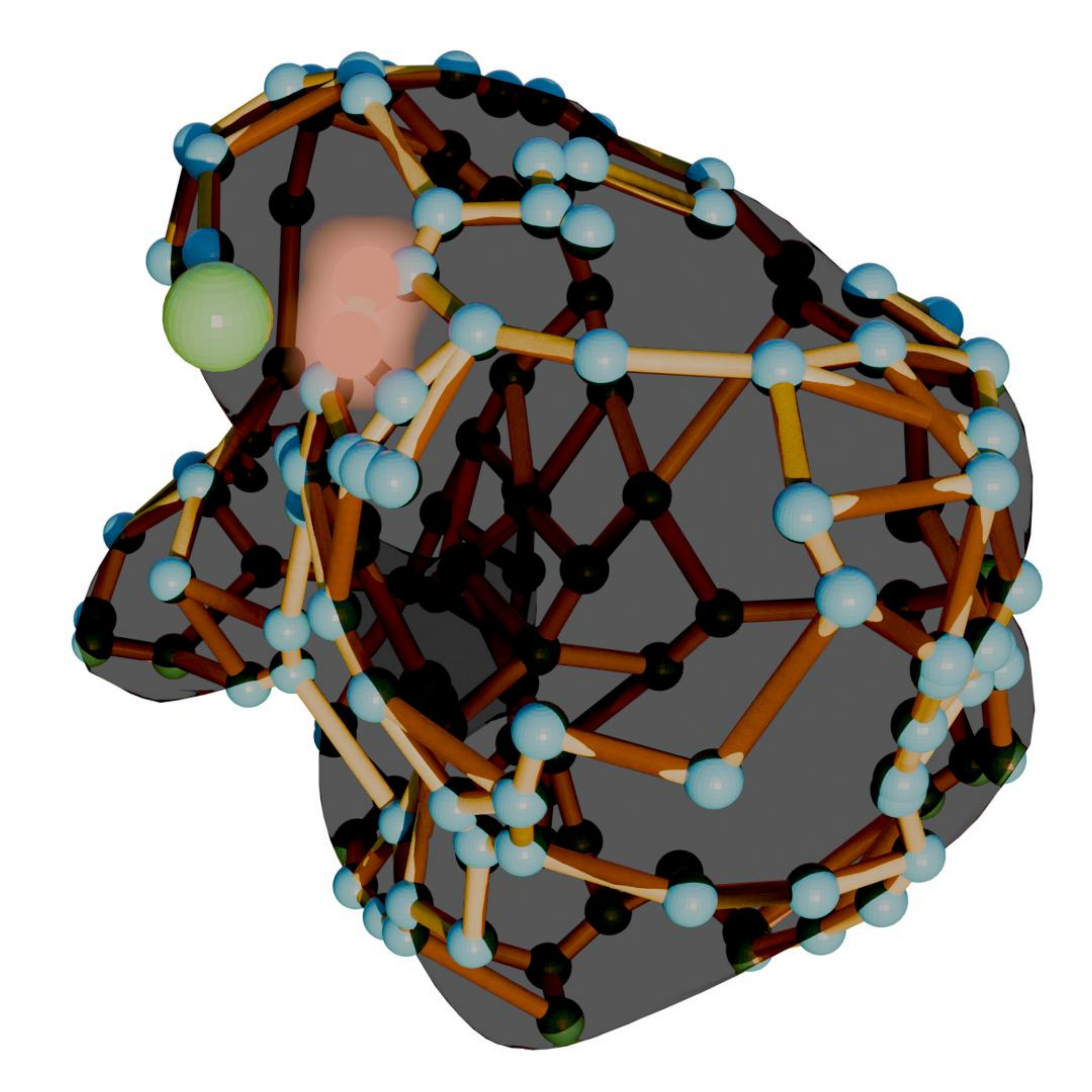}
\caption{An illustration of the search path formed by \cref{alg:beam} on each ``level'' of the HNSW data structure.
The large green sphere denotes the query point, and the search path is highlighted in beige.
}
\label{fig:hnsw-demo}
\end{figure}

The routine for insertions is actually quite similar:
To insert a point $p$ into an HNSW data structure containing a set of points $P$,
the first step is to find the (approximate) $k'$-nearest-neighbors $S$ of $p$ in $P$.
Then, some local updates are made to the graph to incorporate $p$,
in a special (relatively fast) routine.
For our purposes, we need not discuss the details of this routine,
and we point curious readers to the original paper~\cite{malkov2018efficient}.
It should be noted that \citet{manohar2024parlayann}
made the observation that a careful sequence of bulk insertions
are often more efficient when using parallelism.

HNSW has three key parameters, which are typically tuned based on the dataset and desired results:
\begin{itemize}
\item \texttt{ef\_build} is the value $k$ to use for queries at insertion-time (a build-time parameter).
\item \texttt{M} controls the sparsity of the final graph --- it is the maximum number of outgoing edges at each vertex (a build-time parameter).
\item \texttt{ef\_search} is the value $k$ to use for queries at query-time (a query-time parameter).
\end{itemize}
Larger values of \texttt{ef\_build} require longer build times, but usually offer better query time/accuracy tradeoffs.
\texttt{M} is intended to be representative of the ``intrinsic dimensionality'' of the dataset,
in a sense that is commonly applied to manifold learning techniques (see e.g., \cite{BelkinN01}).
Finally \texttt{ef\_search} allows for the tuning of the tradeoff between query time and accuracy.
\citet{malkov2018efficient} present several other parameters and suggest methods for choosing them based on these three.

\section{Initialization and Black-Box Reassignment}
\label{sec:init-blackbox}

In this section, we first run some experiments
for alternative initialization methods with a fairly large value of $k$.
We conclude that the typical initialization methods applied to small values of $k$
are not very effective for large values of $k$,
and that the important step to obtaining good solutions is Lloyd's algorithm.
Afterwards, we will present a naive formulation of Lloyd's algorithm
using a black-box ANNS data structure.
We will then test this formulation over a suite of ANNS data structures.

To aid in the reading of our plots, in each legend,
entries are sorted by their best score.
This is true of \emph{every plot in the paper}.

\subsection{Initialization Techniques for Large $k$}
\label{subsec:init}
Classic formulations of Lloyd's algorithm typically use uniform sampling to
choose the initial centroids
(see \cref{subsec:kmeans-bg} for further discussion).
A key observation of typical applications of $k$-means (i.e., when $k$ is small)
is that the initialization method can be quite important~\cite{av-kmeanpp-06,bmvkv-skmeanpp-12}.
However, our testing indicates that the same does not hold for larger values of $k$.
Even for the relatively small value of $k=10\,000$,
we observe that
all initialization methods
appear to converge to a near-identically-scored solution quite quickly
on all of our tested datasets:
See \cref{fig:example-inits-result},
where we plotted the score over Lloyd iterations after
various initialization methods.
We used
the cuML~\cite{cuml} implementations of \texttt{k-means++} and \texttt{k-means||},
as well as the SciKit-Learn~\cite{scikit-learn} \texttt{k-means++} implementation
in the case of one dataset due to VRAM constraints.
Note that there are some subtleties with some implementations
of these methods~\cite{GrunauORT23}.

\begin{figure*}
\centering
    \begin{subfigure}{0.32\textwidth}
        \centering
        \includegraphics[scale=0.52]{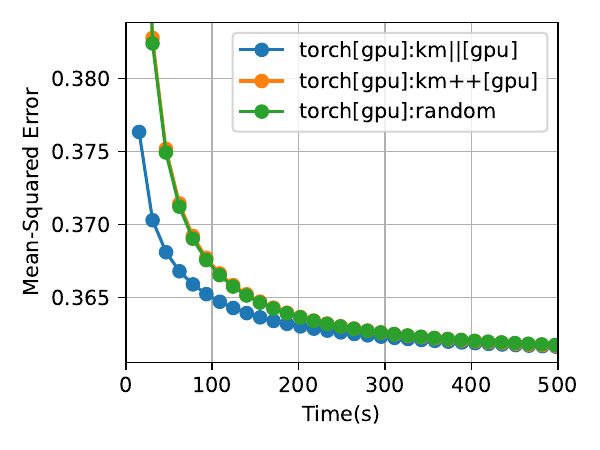}
        \caption{Text2Image10M}
    \end{subfigure}
    \begin{subfigure}{0.32\textwidth}
        \centering
        \includegraphics[scale=0.52]{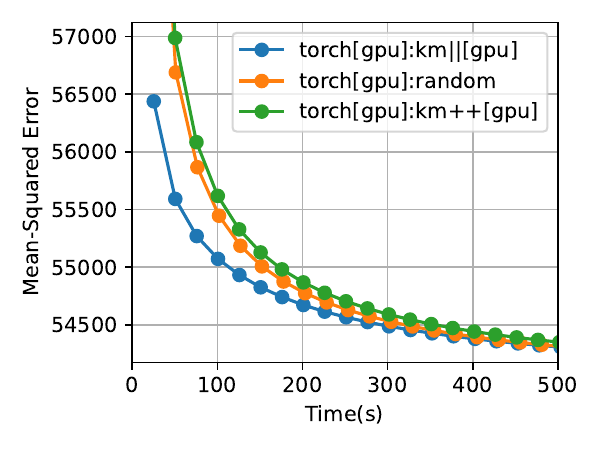}
        \caption{SIFT20M}
    \end{subfigure}
    \begin{subfigure}{0.32\textwidth}
        \centering
        \includegraphics[scale=0.52]{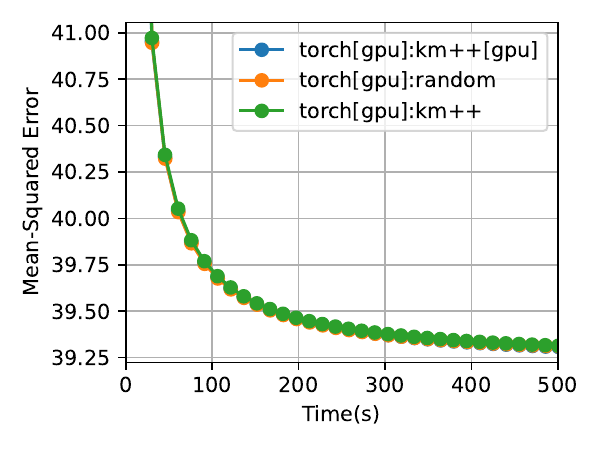}
        \caption{DPR5M}
    \end{subfigure}
    \caption{Comparisons of different initialization methods for $k$-means with $k=10\,000$.}
    \label{fig:example-inits-result}
\end{figure*}

\subsection{Black-Box ANNS for Reassignment}
\label{subsec:black-box-exp}

We suggest a natural modification of Lloyd iterations using
in-memory approximate nearest-neighbor search:
\begin{description}
    \item{\textbf{Build}}: Compute an ANNS data structure over the centers.
    \item{\textbf{Reassign}}: Use the data structure to compute the approximate nearest center for each point in the dataset, and assign the point to the corresponding cluster.
    \item{\textbf{Recompute}}: Recompute the centers as the centroids for the contents of each cluster.
\end{description}
The recompute step remains unchanged from Lloyd's algorithm.
In the low-dimensional setting,
an exact nearest-neighbor search data structure
(such as a k-d tree)
can be used in this framework to give an \emph{exact}
speedup to Lloyd's algorithm.
For the high-dimensional setting we consider,
such methods are inefficient,
and we instead apply ANNS data structures.

\begin{figure}[h]
    \centering
    \includegraphics[scale=0.52]{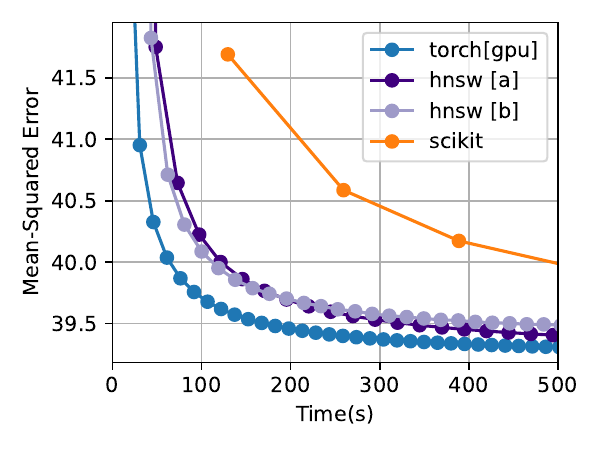}
    \caption{Comparison of HNSW as a black-box method for $k$-means clustering vs Lloyd's algorithm
    on the DPR5M dataset, with $k=10\,000$.
    Initialization is uniformly random.
    }
    \label{fig:blackbox-results-preview}
\end{figure}

We ran experiments on a large suite of
popular in-memory ANNS algorithms,
mostly by leveraging the implementations in the FAISS library~\cite{faiss-24}.
For baselines, we used three different implementations of Lloyd's algorithm:
The (CPU) SciKit-Learn implementation~\cite{scikit-learn},
a simple (GPU) implementation of our own using PyTorch~\cite{torch-19},
and
the (GPU) cuML implementation~\cite{cuml}.
We give a detailed overview
of our experimental setup in \cref{sec:exp-setup},
and a more detailed overview
of the suite in \cref{sec:results-full},
alongside more detailed plots.

Some of the algorithms we tested performed quite well
(e.g., hnsw [b] in \cref{fig:blackbox-results-preview}).
Additionally it turns out that one very simple change to the reassignment step can also
lead to marginally better results for all such methods:
When performing a reassignment of a point,
ensure that the new center is closer than the previously assigned center.
After performing this change, we did obtain some slightly improved results
(e.g., hnsw [a] vs hnsw [b]
in \cref{fig:blackbox-results-preview}).
but our previous observations remain true.

For a full list of algorithms we tested,
see \cref{sec:results-full}.
The main observations are as follows:
\begin{itemize}
    \item For sufficiently large values of $k$, most methods based on clustering
    and/or quantization generally exhibited
    comparable performance to SciKit-Learn's implementation~\cite{scikit-learn} of Lloyd's algorithm. Generally these methods performed iterations much more rapidly than Lloyd's algorithm (as expected),
    but improved in score at a much slower rate over time.
    \item In contrast, the search-graph methods we tested
    were generally quite effective.
    Moreover,
    they appeared to obtain even better comparative performance
    compared to Lloyd's algorithm as $k$ increased.
    In a handful of cases, with large values of $k$, HNSW in particular was marginally better than the \emph{GPU} implementations of Lloyd's algorithm (specifically, the cuML implementation~\cite{cuml}, and a basic implementation of ours leveraging PyTorch~\cite{torch-19}).
\end{itemize}

While these results are extremely promising,
this black-box approach is (in most cases)
not more practical
than simply running Lloyd's algorithm on a GPU,
even if it is comparable in some cases.
We aim to arrive at a more practical method,
so we will discuss our more carefully-designed methods in the next section.

\section{Bulk Seeded Approximate Nearest-Neighbor Search via Search-Graphs} \label{sec:methodology}
Although approximate nearest-neighbor search
seems to fit in nicely into the framework of Lloyd's algorithm,
naively using it does not result in good practical performance.
We believe that this is because it does not take advantage of all available information.

Instead, we argue that the correct problem to solve
is one we will call
\defn{Seeded Approximate Nearest-Neighbor Search} (SANNS),
as well as a semi-offline variation of the problem
we will call
\defn{Bulk Seeded Approximate Nearest-Neighbor Search} (BSANNS).
Similarly to ANNS,
in SANNS,
the goal is to design a data structure that
takes as input a set $P$ of points (say $P\subset\Real^d$ for simplicity),
and a pairwise distance function (say Euclidean distance for simplicity).
Then, the data structure must answer queries
consisting of a query point $q$
as well as identifiers
for some small set of points in $P$.
This small set of points in $P$ is called the set of \defn{seed points}.
This is intended to be a \emph{learning-augmented}
form
of ANNS:
It is \emph{not}
guaranteed that the seed points
are good approximate nearest-neighbors of $q$.
That is, algorithms for this problem should work
regardless of whether or not the seed points provide useful information (in the language of learning-augmented algorithms, they should be \defn{robust})
However,
such algorithms should provide better results
(in the tradeoff between time and result accuracy)
if the seed points happen to be decent approximate neighbors of $q$ (in the language of learning-augmented algorithms, they should be \defn{consistent}).
See \cite{learnaug} for an overview of learning-augmented algorithms.
In the batched version of this problem (BSANNS),
the only difference is that the queries are given in large batches
(say, of size $~|P|$).
The batched version of this problem is related to the so-called ``approximate all-$k'$-nearest-neighbor search'' problem, which is an offline version of ANNS
(see e.g.~\cite{MaL19} for an algorithm in the low-dimensional exact version of this problem).
In the context of $k$-means clustering,
the ``semi-offline'' variation with batches of our dataset is more useful than a fully offline version
because it allows us to minimize RAM usage.

Additionally, in \cref{sec:provable-sanns},
we observe that a previously-studied
search-graph algorithm with theoretical guarantees for ANNS
naturally extends to SANNS as well,
with guarantees
for robustness and consistency
(the expected types of guarantees for learning-augmented algorithms),
based on analysis by \citet{indyk2023worst}.
Unfortunately, this algorithm is impractical for our use-case of $k$-means clustering with large $k$,
since the build routine would take
$\Omega(k^3)$ time in this context.

\subsection{Seeded Search-Graphs}
\label{subsec:seeded-search-graphs}

Recall that search-graph methods
use the greedy/beam search routine
in \cref{alg:beam}.
In particular, they use a prescribed initial point,
whose choice depends on the particular search-graph method.
We can additionally modify this routine to use \emph{multiple}
initial points,
so long as the additional points are not too numerous.
This leads us to a candidate method for SANNS:
\emph{In \cref{alg:beam}, use the provided seed points as additional initial points in the greedy search routine.}
We note that, for HNSW in particular,
these additional initial points should only used during the search at the bottom-most layer,
since such points may not even exist at higher layers.

\paragraph{Bulk queries for additional seed points}
To approach the BSANNS problem while leveraging our techniques for SANNS,
we propose a heuristic:
Specifically, we introduce \emph{additional}
seed points for bulk queries.
By grouping together correlated queries from a bulk query,
we can then perform queries for each group in a careful fashion.
Specifically, while iterating through each group,
we obtain the top results for each query, and then use
them as additional seeds for the next query.
We also propose a simple method for choosing an iteration order for the datapoints within each group: Randomly project all datapoints into a $1$-dimensional space,
and sort them.

The idea here is simple: Correlated query groups can be considered a very ``rough'' clustering of the data (not necessarily a $k$-means clustering),
and so they are more likely to be assigned to the same final centroid.
Moreover, within each group, the ordering is expected to project each group of $n$ points into $\Real$ with
$O(\log n)$ distortion, as shown by \citet{johnson1984extensions}.
With HNSW, the correlation technique we used was simply to group everything by its default ``initial point'' given by the recursive structure of HNSW.
One could also consider grouping everything by its best seed point,
but we would only expect this to be effective in cases where the size of a bulk query is much larger than the number of datapoints (so that the groups are of nontrivial size).

\subsection{Seeded Search-Graphs for $k$-Means}
\label{subsec:k-means-ssg}

We now discuss how to apply seeded search-graphs
(and more broadly, algorithms for SANNS/BSANNS)
to Lloyd's algorithm,
and some further specialized improvements.
In particular, we will implement
our method using HNSW as a basis,
with several layered improvements.

\paragraph{Using Seed Points}
The first method we employed was to
to use the previous iteration's assignments as seed points for seeded search-graphs.
There is reason to believe this is a good heuristic:
The centers slow their movements over the course of many Lloyd iterations
(see \cref{tab:centroid-movement}),
so the best approximate assignment is increasingly likely to be the previous assignment
as the number of iterations increases.

\renewcommand{\arraystretch}{1.1}
\begin{table}[h]
    \centering
    \small
    \begin{tabular}{r|r||r|r||r|r}
         It.& Avg. dist. &
         It.& Avg. dist. &
         It.& Avg. dist. \\
            \hline
 1 & 136.012  &  14 & 4.070  &  27 & 1.889 \\ 
 2 &  31.827  &  15 & 3.780  &  28 & 1.842 \\ 
 3 &  19.410  &  16 & 3.529  &  29 & 1.766 \\ 
 4 &  14.190  &  17 & 3.302  &  30 & 1.682 \\ 
 5 &  11.243  &  18 & 3.096  &  31 & 1.591 \\ 
 6 &   9.357  &  19 & 2.899  &  32 & 1.519 \\ 
 7 &   8.044  &  20 & 2.766  &  33 & 1.456 \\ 
 8 &   7.082  &  21 & 2.587  &  34 & 1.375 \\ 
 9 &   6.282  &  22 & 2.432  &  35 & 1.295 \\ 
10 &   5.701  &  23 & 2.290  &  36 & 1.239 \\ 
11 &   5.160  &  24 & 2.182  &  37 & 1.237 \\ 
12 &   4.732  &  25 & 2.084  &  38 & 1.172 \\ 
13 &   4.408  &  26 & 1.969  \\
    \end{tabular}
    
    \caption{The average distance centroids move during each standard Lloyd iteration while clustering SIFT1M ($10^6$ points) with $k=5000$ clusters,
    demonstrating that the movement of the centroids slows over time.
    The average distance from a datapoint in SIFT1M to the respective closest of these centroids
    ranges from approximately $43\,481$ to $46\,524$ over $38$ iterations, showing that these movements are small.
    }
    \label{tab:centroid-movement}
\end{table}

\begin{figure*}[h]
    \centering
    \begin{subfigure}{0.3\textwidth}
    \centering
    \includegraphics[scale=0.50]{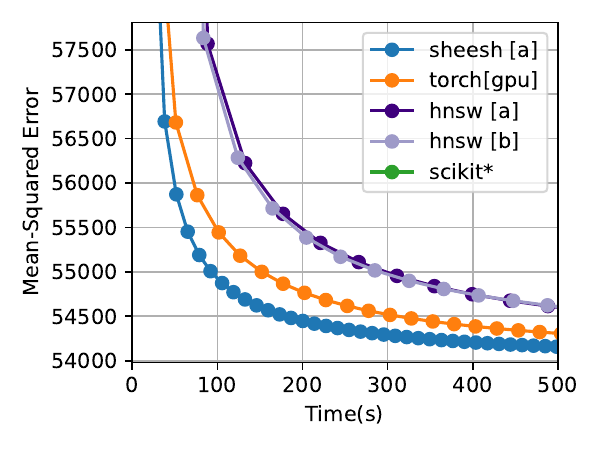}
    \caption{SIFT20M, $k=10\,000$}
    \end{subfigure}
    \begin{subfigure}{0.3\textwidth}
    \centering
    \includegraphics[scale=0.50]{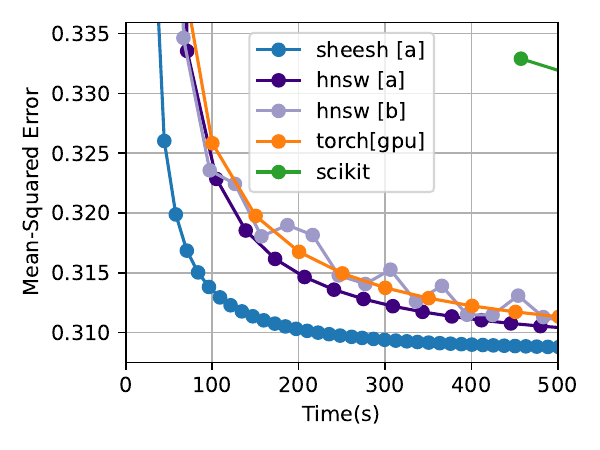}
    \caption{Text2Image10M, $k=50\,000$}
    \end{subfigure}
    \begin{subfigure}{0.3\textwidth}
    \centering
    \includegraphics[scale=0.50]{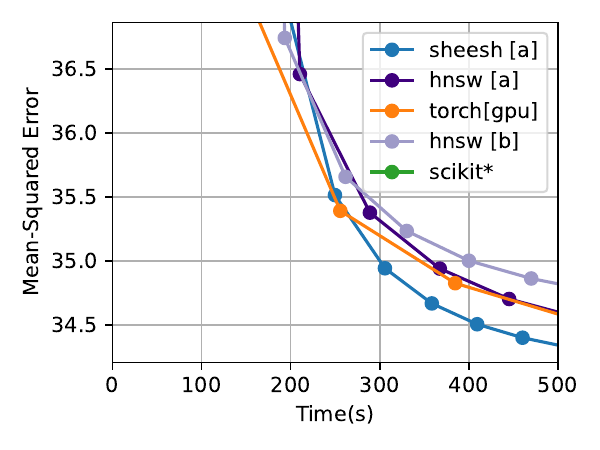}
    \caption{DPR5M, $k=100\,000$}
    \end{subfigure}
    \caption{Comparison of our approach with GPU acceleration, as well as the black-box HNSW approach on the SIFT20M, Text2Image10M, and DPR5M datasets respectively, for $k=10\,000$. Initialization is uniformly random.}
    \label{fig:results-preview}
\end{figure*}

It turns out this often already provides a small improvement
(this comparison is shown in \cref{sec:results-full}, alongside many others).
However, we can obtain even better performance as follows:
In each iteration, instead of recording just the best centroid assignment,
we can actually record several of the top assignments.
We use the top $10$, although this parameter has not been tuned.
Then, instead of using a single seed point during the next iteration,
we can them all as seed points,
by initializing the sets $C,N$
in \cref{alg:beam}
with these seed points
($N$ obtaining only the nearest $b$ of the seed points).
This leads to an even further improved algorithm
(this comparison is also shown in \cref{sec:results-full}).

\begin{algorithm}[h]
   \caption{\defn{SHEESH}: Accelerated Lloyd Iteration with BSANNS via Seeded Search-Graphs}
   \label{alg:lloyd-bsanns}
\begin{algorithmic}
   \STATE {\bfseries Input:} $P\subset\Real^d$, centers $C\subset\Real^d$, $|C|=k$,
   previous multi-assignments $S:P\to 2^P$,
   previous search-graph data structure $D$
   \STATE {\bfseries Build:} Build a new search-graph data structure $D'$ by using $D$.
   \STATE {\bfseries Reassign:}
   \FOR{each chunk $U$ of $O(k)$ points in $P$ (in parallel)}
   \STATE Group the points of $U$ into roughly-correlated groups.
   \STATE Randomly project each group into $\Real^1$, and sort the projected group.
   \FOR{each group $G$ of $U$}
   \FOR{each point $p$ of $G$, in the sorted order}
   \STATE Let $q$ be the previous point.
   \STATE Use $S'(q)\cup S(p)$ as seeds.
   \STATE With all these seeds, compute the seeded approximate $\sim10$ nearest centers of $p$
   using $D'$, and save the results as $S'(p)$.
   \ENDFOR
   \ENDFOR
   \ENDFOR
   \STATE {\bfseries Recompute:} Compute the new centers $C'$ as centroids.
   \STATE {\bfseries Output:} New centers $C'$, new multi-assignments $S'$,
   new search-graph data structure $D'$
\end{algorithmic}
\end{algorithm}

\paragraph{Continuous Rebuilds}
As noted by each of the works
using the so-called ``inverse assignment'' method
discussed in
\cref{subsec:related},
it is often inefficient
to build a data structure over the centers
from scratch at each iteration.
The inverse assignment method is one way of handling this
issue, although,
as noted, it does not scale to large datasets.
Instead, we suggest the following approach:
On all iterations except the first,
leverage the search-graph of the previous iteration
to construct the new search-graph,
rather than starting with an empty graph.
Since most centers do not significantly change between later iterations on average (see \cref{tab:centroid-movement}),
this graph serves as a good coarse approximation of the desired search-graph,
and the quality of the approximation improves over time
as the centers simultaneously converge.
For HNSW in particular, we note that
we fix the subsampling of the centers at higher levels of the data structure ---
allowing the same simultaneous convergence to occur on the higher levels of the search-graph.

\paragraph{Min Iteration Count}
Intuitively, search-graphs seeded with result from a previous Lloyd iteration
may get stuck more easily in the exact same local optimum
across iterations,
especially if the graph itself is less prone to change over time.
To combat this, we suggest a heuristic:
Rather than always terminating \cref{alg:beam} early
once a local optimum is reached,
we also require that a specified minimum number of iterations have been performed.
In particular, for HNSW, this bound is only applied at the lowest level of the hierarchy
(the search-graph over the full dataset), where the seeds are also applied.

The combination of these two additional improvements,
along with the bulk methodology discussed
\cref{subsec:seeded-search-graphs},
is an algorithm we call
\defn{SHEESH}
(\underline{S}eeded searc\underline{H}-grap\underline{H}s for $k$-m\underline{E}ans clu\underline{S}t\underline{E}ring).
It is labeled sheesh [a] in \cref{fig:results-preview}.
Recall that, in all of our plots' legends,
entries are in descending order by their best score.
In addition, any plot entries marked with a ``*'' timed out
by taking $>500$ seconds to finish even a single Lloyd iteration.
In particular, our algorithm with all features enabled (sheesh [a]),
achieved the best score in all experiments.
We have run numerous experiments in addition to the ones shown
in \cref{fig:results-preview},
which we discuss more thoroughly in \cref{sec:results-full}.
Our full algorithm for $k$-means leveraging seeded search-graphs is outlined in \cref{alg:lloyd-bsanns}.

\section{Conclusion}
\label{sec:conclusion}

We have presented a methodology for accelerating $k$-means with large values of $k$.
In particular, we leveraged methods previously applied to ANNS,
and improved them for SANNS and the specific application of $k$-means clustering,
culminating in \cref{alg:lloyd-bsanns}.
We have
demonstrated that our method is quite performant for $k$-means clustering
some large high-dimensional image and text embedding datasets.
We believe there are two main interesting directions of further research:
Evaluating our methodology on various applications of $k$-means clustering,
and
Studying methods to further accelerate our methods.
For the latter reason in particular,
we believe the results in \cref{subsec:black-box-exp}
are quite interesting, since they do not agree with the performance
of the same techniques for standard in-memory ANNS.
This may indicate that there are more interesting specialized methods that can
be applied to SANNS in the context of $k$-means.
One avenue for exploration could be to
try adapting similar techniques to what is used for
so-called ``out-of-distribution'' ANNS, such as \citet{ChenZHJW24}.

Another avenue for accelerating our methods
could be to explore approaches based on specialized hardware.
We focused on devising CPU-based algorithms,
and our final methods rely on search-graphs.
There is some work exploring GPU-based algorithms for
ANNS on search-graphs~\cite{ZhaoTL20,YuWZQZL22,GrohRWL23,OotomoNNWFW24}.
Currently, the fastest of these is CAGRA~\cite{OotomoNNWFW24}.
Unfortunately, CAGRA's preprocessing routine starts by applying
a (hardware-accelerated) quadratic-time algorithm.
Moreover, one of its key
hardware-leveraging methods
(step 0 of their described algorithm)
seems unlikely to be helpful for SANNS.
Consequently, we believe it is likely that CAGRA's
hardware-acceleration methods would not obtain the same speedup over CPU,
although this would need to be tested.
Intuitively, searching over a graph (even a regular graph)
is not a very efficient operation for a GPU,
so there may be other more-effective families of methods for hardware acceleration.

\section*{Acknowledgements}

Research of the third author supported in part by an NSERC PGSD. This material is based upon preliminary work supported by the Google Cloud Research Credits for PhD students.

\section*{Impact Statement}

This paper presents work whose goal is to advance the field of 
Machine Learning. There are many potential societal consequences 
of our work, none which we feel must be specifically highlighted here.

\bibliography{references}
\bibliographystyle{icml2025}

\newpage

\appendix
\onecolumn

\section{Further Background}
\label{sec:background-additional}

In this section, we discuss the additional background
information omitted from \cref{sec:background}.

\subsection{$k$-Means}
\label{subsec:kmeans-bg}

A typical implementation of $k$-means clustering in a larger
application (for example, SciKit-Learn \cite{scikit-learn})
would be the following:
\begin{enumerate}
    \item (Optionally) use a sub-sampling technique to reduce the dataset size.
    \item Initialize the centroids.
    \item Use a local search technique to improve the solution.
    \item Terminate after a pre-specified number of iterations or amount of time.
\end{enumerate}

Each of these steps has a variety of avenues for improvement, see
the survey by \citet{kmeans-survey-23}.
For each, we will summarize only some of the most important methods potentially relevant to our case of large $k$.
In particular, sub-sampling is quite limited
for large $k$, since the number of points per cluster may already be quite small.

\paragraph{Existing work on initialization}

The easiest and most efficient initialization method is to
uniformly sample points.
However,
strong initialization methods are often desirable
since they can obtain provable approximation ratios.
Some particularly popular and easily implementable algorithms include
\texttt{k-means++}~\cite{av-kmeanpp-06}, \texttt{scalable k-means++}/\texttt{k-means||}~\cite{bmvkv-skmeanpp-12}, and \texttt{multi-swap k-means++}~\cite{beretta2024multi}).
However,
even the theoretically fastest of these algorithms (\texttt{k-means||}) takes time at least $\Omega(k^2)$
(a tighter lower bound in terms of some slightly different parameters is given in \cite{bmvkv-skmeanpp-12}).
For large values of $k$, this is still far slower than uniformly random initializations.
Standard implementations support this observation:
SciKit-Learn's~\cite{scikit-learn} implementation of \texttt{k-means++}, the reference implementation of \texttt{k-means||}~\cite{bmvkv-skmeanpp-12}, and the cuML implementation of each~\cite{cuml},
all seem to be fairly slow on even moderately sized datasets.

\paragraph{Approximation guarantees}
There has been prior work on providing approximation guarantees for the $k$-means problem using techniques related to local search. One example is a
$(9+\epsilon)$-factor approximation algorithm by \citet{kanungo2002kmeans},
and a PTAS when the dimension $d$ is fixed by \citet{friggstad2019ptas}.
It is also known that there is no PTAS for arbitrary dimension \cite{awasthi2015hardness}.

\paragraph{Existing work on Local Search Methods}
We previously mentioned Lloyd's algorithm as a popular local search method
for $k$-means clustering.
There are also other local-search methods that have been well-studied.
However, Lloyd's algorithm
has remained standard
and has continued to show strong results in practice,
particularly for high-dimensional datasets.
One reason for this is that Lloyd's algorithm is highly parallelizable.
See the survey by \citet{kmeans-survey-23}
for an overview.
That said, all techniques we will discuss in this work could be generalized
to many variations of Lloyd's algorithm
clustering that subsample data points in each step,
but studying the effectiveness of such techniques is more difficult.

Although Lloyd's algorithm obtains good results in practice, it is known that there are two-dimensional datasets~\cite{v-kmemip-09} for which Lloyd's algorithm takes an exponential number of iterations to converge.
Moreover, without a careful initialization, it may
produce an arbitrarily bad clustering~\cite{av-kmeanpp-06}.
Practical implementations often use a time limit or an iteration limit (see implementations in the popular SciKit-Learn~\cite{scikit-learn} and FAISS~\cite{faiss-24} libraries),
instead of waiting for convergence.

A natural question is whether the execution speed of Lloyd iterations can be improved,
which is essentially equivalent to asking if step \ref{step:reassign} (assignment) can be accelerated.
This question has been studied in several contexts.
For low-dimensional data,
various methods are quite effective at accelerating this step exactly, including the use of k-d trees~\cite{kmnpsw-ekmca-02}, and methods based on the triangle inequality~\cite{Elkan03,Hamerly10,hd-alakmc-15}%
\footnote{For a comprehensive discussion of these techniques, their limitations in high dimensions, and other methods of acceleration (like the use of parallelism), see the book chapter by \citet{hd-alakmc-15}. }.
However, in many practical applications, we would like to run $k$-means in higher dimensions for applications such as natural language processing (e.g. word2vec recommends 100 to 300 dimensions~\cite{MikolovCCD13}, and the dense retriever model of \citet{karpukhin-etal-2020-dense} uses 768 dimensions) and neural network embeddings (e.g. the image embeddings of \citet{HuSASW20} are 154-dimensional).
Unfortunately, there are known lower bounds for exact techniques to accelerate this step in high-dimensional datasets~\cite{BorodinOR99}.
A few works have attempted to bypass exact nearest-neighbor search
by leveraging classical techniques for approximate nearest-neighbor search.
We discuss these works in
\cref{subsec:related}.

\subsection{Related Work}
\label{subsec:related}

Compared to the vast literature on $k$-means clustering as a whole,
we are only aware of a handful of works that have attempted to apply
any form of approximate nearest-neighbor search to any form of clustering.

\paragraph{Works using ANNS to Accelerate Lloyd's algorithm}
A few existing works have applied
methods for approximate nearest-neighbor search
in a black-box fashion
to accelerate Lloyd's algorithm (or variants)
in various contexts of $k$-means clustering.
Note that we will omit discussion of methods
that are essentially just dimension reduction techniques,
for which there are many works.

Several of these use a similar approach to
the general ``black-box'' methodology
we thoroughly test in
\cref{subsec:black-box-exp}.
\citet{PhilbinCISZ07} presented a method
greedily traversing randomized k-d trees as an ANNS heuristic for this purpose,
but they do not test their solution
w.r.t. $k$-means clustering score
(although some later work uses their solution as a baseline).
\citet{GongPYBBF15} applied techniques for
locality-sensitive hashing
(LSH)
--- which is a subclass of
space-partitioning methods (see the survey by \citet{lshsurvey}) --- 
to binary data under Hamming distance
with ``mini-batch'' $k$-means local search
(\citet{Sculley10} gives a discussion of mini-batches,
which can be seen as a modified form of Lloyd's algorithm
that maintains parallelizability).
\citet{HuWBZC17} present a similar method,
instead using Hamming LSH with a ``reranking'' step
to cluster Euclidean data via binary code quantization.
Note that ANNS over Hamming distance is generally easier than
Euclidean distance, since it a special case.
Moreover, locality-sensitive hashing
is now significantly outperformed
by modern ANNS techniques in practice~\cite{aumuller2020ann},
so this is likely not the most effective use of this black-box approach
(as we will see in \cref{subsec:black-box-exp}).
As part of an implementation
for a variant of ``Product Quantization'',
\citet{baranchuk2018revisiting} applied
HNSW to $k$-means clustering over
relatively small chunks of data.
Their method for doing so
is similar to our initial ``black-box'' methodology
presented in
\cref{subsec:black-box-exp},
although they focus on a much smaller-scale case,
and they did not provide any empirical justification
for their choice of HNSW over other ANNS methods,
nor did they empirically test their methods w.r.t. the $k$-means clustering objective
(nor was obtaining a good clustering score their goal).

A few works have also explored a different method of applying ANNS techniques
to the assignment stage of Lloyd iterations:
Instead of assigning dataset points in $P$ to their nearest center in $C$,
the centers can instead ``flood fill'' the dataset
using certain types of ANNS data structures
constructed over $P$ instead of $C$.
We will call this the \defn{inverse assignment} method.
\citet{kmnpsw-ekmca-02} applied the inverse assignment method
to compute \emph{exact} assignments.
\citet{AvrithisKAE15}
applied the inverse assignment method
by essentially projecting into a quantized two-dimensional
space (thereby doing a quantized dimension reduction).
\citet{WangWKZL15}
also employ a variation of the inverse assignment method
by constructing an approximate neighborhood graph,
which they then leverage to prune distance computations.
In particular, they construct their graph
using
random-projection trees~\cite{DasguptaF08},
an ANNS space-partitioning method.
Unfortunately,
the inverse assignment method is limited to
(small) data sets that can fit wholly in-memory,
since it requires a more careful traversal of
a specialized ANNS data structure built for $P$, rather than $C$.
This is in contrast to methods that build structures
over the
``forward'' assignment method
(including ours),
which only need to build and store a structure
for $k$ points in the metric space ---
even for excessively large values of $k$ (e.g. $k=|P|/100$),
this is still a very significant difference
in memory usage for massive datasets.
There is one method for which this limitation can
be overcome for the inverse method:
\citet{MatsuiOYA17} suggest
using product quantization
on the input vectors.
Since this is a quantization method,
rather than another form of ANNS method,
applying the inverse method
does not actually prune any distance computations,
but rather just speeds them up individually.
The experimental results of \citet{MatsuiOYA17}
suggest that, although faster,
their method cannot achieve the same score
as Lloyd's algorithm,
very quickly arriving at poor local optimums even for quite small values of $k$.
In particular,
since their method amounts to a brute force
with quantization methods,
this suggests that any similar approaches
applying an ANNS technique to Lloyd's algorithm
involving any sort of quantization is likely to result
in poor local optimums.
As we will see, this appears
to be true in our results as well.

Compared to all of these approaches,
we present a more complete analysis
of ANNS methods
for Lloyd's algorithm,
and we furthermore determine that they
are not effective without further work.
Moreover, we complete this further work,
eventually devising seeded search-graphs.

\paragraph{Other Forms of Clustering}
To the best of our knowledge, only one academic work has studied the application of approximate nearest-neighbor search to large-scale clustering that is not $k$-means clustering:
PECANN~\cite{yu2023pecann,yu2024parallel} studies the application of black-box approximate nearest-neighbor search methods to hierarchical density-based clustering.

Although not an academic work,
the software library USearch~\cite{usearch} can produce a hierarchical clustering using HNSW,
although the developer has not
publicized any experimental results or detailed documentation on their methodology,
and the feature is still marked as in-development.
Their method appears to involve treating each point on a non-zero level as a cluster ``center'',
and performing a simultaneous flood fill from all points in a non-zero level to the points of the dataset in the level below.
This would be similar to performing a random sample initialization of $k$-means, with no local search iterations,
and performing the ``assignment'' stage in a way so as to perform graph-based clustering
(i.e., approximating geodesic distance over a manifold approximated by the graph).

\newpage
\section{Experimental Setup}
\label{sec:exp-setup}
We ran multiple rounds of experiments, so in this section, we present the shared details of our experimental setup.

\paragraph{Environment}
We conducted the experiments on
a workstation machine with Ubuntu 22.04.5 LTS,
equipped with an AMD Ryzen 9 7950x CPU, 64GB of RAM, an Nvidia RTX 3090 GPU, and datasets stored in a 2TB SSD.
Note that this large amount of RAM is primarily for testing baseline code and existing libraries --- our own methods will not use a significant amount of RAM for any dataset, and (most importantly) will not require storing the entire dataset in memory at any time.
Note also that, at the time of writing, this GPU is significantly more expensive (by a factor of $~2\times$ in most marketplaces) than the CPU.

All timed CPU-based algorithms were allowed to use a maximum of $12$ threads
to reduce possible conflicts with operating system processes.
No limits were placed on the algorithms using the GPU.

\paragraph{Software Libraries}
We will make use of the following libraries for various baselines:
\begin{itemize}
    \item SciKit Learn~\cite{scikit-learn} on CPU
    \item cuML~\cite{cuml} on GPU
    \item PyTorch~\cite{torch-19} on GPU
    \item FAISS~\cite{faiss-24} on CPU
\end{itemize}
We use SciKit Learn, cuML, and PyTorch
for reference implementations of Lloyd's algorithm
and initialization routines.
In particular, SciKit Learn and cuML both offer
built-in implementations of Lloyd's algorithm and initialization methods,
while PyTorch enabled us to write a straightforward, short, and highly-efficient GPU-based
implementation of Lloyd's algorithm.
We use FAISS for its implementations of various baseline approximate nearest-neighbor search routines.
Note that none of these libraries are used for our own code, for which we discuss implementation details in \cref{sec:impl-deets}.

\paragraph{Datasets}
We tested on two image-embedding datasets,
and one text-embedding dataset:
\begin{itemize}
    \item Yandex's \textbf{Text2Image10M} dataset \cite{yandex-ann-dataset-21} which consists of images embeddings produced by the \texttt{Se-ResNext-101 model} \cite{HuSASW20}.
    This data set was used for benchmarking for the NeurIPS 2023 large-scale ANNS competition~\cite{bigann2023}.
    Typically, this is used as a cross-modal data set, with an ANNS query set derived from text embeddings, but in this paper we use Text2Image10M purely for its image embeddings, since clustering these points is what one would do for our suggested application (see \cref{para:out-of-core-anns}).
    \item
    The \textbf{SIFT1B} dataset~\cite{jtda-sobv-11}, and a
    20M slice of the SIFT1B
    dataset that we will call \textbf{SIFT20M}. These datasets are 128-dimensional image descriptors in SIFT format~\cite{Lowe04}. 
    This dataset is also frequently used for
    benchmarking large-scale out-of-core ANNS algorithms~\cite{baranchuk2018revisiting,JohnsonDJ21,jayaram2019diskann}.
    \item The \textbf{DPR10M} dataset generated by \citet{aguerrebere2023similarity} from 768-dimensional dense passage retriever model of \cite{karpukhin-etal-2020-dense}. To make the dataset of comparable size to the other ones, we created \textbf{DPR5M} by taking 
    a 5M slice of DPR10M.
    This is a higher dimensional text-based dataset, to contrast with the other datasets.
\end{itemize}

\renewcommand{\arraystretch}{1.1}
\begin{table}[]
    \centering
    \begin{tabular}{c|c|c|c|c}
         Dataset & Type & Dim & Points & Size \\
            \hline
         \textbf{SIFT1B} & Image & 128 & \hfill 1 bil. & \hfill 512.00 GB\\ 
         \textbf{SIFT20M} & Image & 128 & \hfill 20 mil. & \hfill 10.24 GB\\ 
         \textbf{Text2Image10M} & Image & 200 & \hfill 10 mil. & \hfill 8.00 GB \\
         \textbf{DPR5M} & Text & 768 & \hfill 5 mil. & \hfill 15.36 GB
         \\  \hline
    \end{tabular}
    
    \caption{Description of the datasets used in experiments.}
    \label{tab:data_sets}
\end{table}

Since these datasets are quite large,
we are unfortunately unable to distribute them ourselves.
However, we provide instructions for reproduction of these datasets
as part of our code.
For development purposes, we also used \defn{SIFT10K} and \defn{SIFT1M}.
For these smaller datasets, we leveraged the Pooch~\cite{pooch-20} library for retrieval and caching.

We note that our comparison plots use
the $10$-million and $5$-million sized datasets only.
This is primarily due to limitations in our baselines,
rather than any limitations in SHEESH.

\newpage
\section{Full Results}
\label{sec:results-full}

In this section, we present the full set of results for our experiments.

For black-box acceleration, we tested the following families of popular methods:
\begin{itemize}
    \item Baseline: Lloyd's algorithm ($3$ implementations, $2$ of which were on GPU)
    \item Quantization-only techniques: Scalar Quantization~\cite{liu2024vq}, Product Quantization~\cite{matsui2018pq}
    \item Clustering-only techniques: IVF~\cite{SivicZ03}, ScaNN~\cite{guo2020accelerating} (with quantization disabled, amounting to a $k$-means tree)
    \item Combined clustering+quantization techniques: ScaNN~\cite{guo2020accelerating}, IVFPQ~\cite{jegou2010product}, IVFPQR~\cite{jegou2011searching}
    \item Search-graph techniques: NN-Descent~\cite{DongCL11}, HNSW~\cite{malkov2018efficient}, NSG\footnote{%
While we attempted to thoroughly test NSG,
which showed promising preliminary results similar to HNSW,
we were not able to evaluate it in most of our experiments,
for the following reason:
At the time of writing, there appears to be an occasional bug
in the FAISS implementation of NSG
that can result in an infinite loop,
so we disabled it while performing most of our experiments.
It still appears in one of our plots, where it shows good performance.
}~\cite{FuXWC19}
\end{itemize}
Whenever possible, we leveraged implementations given in
FAISS~\cite{faiss-24}.
For ScaNN,
we used the reference implementation,
which itself is called ScaNN~\cite{guo2020accelerating, sun2024soar}.
For a baseline CPU implementation,
we used SciKit-Learn~\cite{scikit-learn}'s CPU implementation of Lloyd's algorithm.
We also compared against two GPU implementations of Lloyd's algorithm:
The GPU implementation in cuML~\cite{cuml},
as well as a simple implementation of our own using PyTorch~\cite{torch-19}.
We note that, in some sense, comparing
the CPU-based approaches
to GPU implementations is unfair,
especially since our GPU is significantly more expensive than our CPU
(as noted in \cref{sec:exp-setup}).
However, practical implementations of $k$-means clustering
very frequently use GPU acceleration (e.g.,~\cite{faiss-24}),
so this is an important comparison to make
when attempting to devise a method to be used in practice.
Moreover, since our experiments suggest that
our final methods are far superior to the GPU accelerated implementations,
they serve as a good reference point.

We had a number of parameters that varied for each black-box technique,
as well as some that varied for our own techniques discussed in \cref{sec:methodology}.
Most notably,
we have the
``avoid\_regress'' parameter
(the simple improvement discussed in
\cref{subsec:black-box-exp}).
For the quantization-only techniques,
we opted not to re-run
with avoid\_regress=true,
since they were excessively slow.
For the ScaNN library,
we tried many combinations of parameters,
including turning on/off quantization.
For our own techniques,
we also had several parameters.
In particular, we toggled
several of the different strategies discussed in \cref{sec:methodology}
to study their effectiveness.
We give a legend of all parameter variations
in \cref{tab:param-var},
which can be used as reference
for all of our experimental plots.

For the algorithm parameters we did not vary,
we generally applied sane/recommended defaults.
In particular, for our methods, as well as the black-box tests with HNSW, we used the following parameters:
\begin{itemize}
\item $\text{\texttt{ef\_build}}=200$
\item $\text{\texttt{M}}=60$
\item $\text{\texttt{ef\_search}}=10\times\text{\texttt{num\_prev\_assignments}}$
\item $\text{\texttt{min\_iterations}}=2\times\text{\texttt{ef\_search}}+1=21$
\end{itemize}
There is almost certainly some improvement to be gained by better tuning these parameters
to each dataset, but we have not done so.

\begin{table}
\centering
\begin{tabular}{ll}
\textbf{Tagged Label} & \textbf{Parameters} \\ \hline
ivf-PQ[a] & avoid\_regress: true \\
ivf-PQ[b] & avoid\_regress: false \\
hnsw[a] & avoid\_regress: true \\
hnsw[b] & avoid\_regress: false \\
ivf-PQr[a] & avoid\_regress: true \\
ivf-PQr[b] & avoid\_regress: false \\
ivf-flat[a] & avoid\_regress: true \\
ivf-flat[b] & avoid\_regress: false \\
scann[a] & \footnotesize{num\_leaves: 200, num\_leaves\_to\_search: 10, use\_score\_ah: true, reorder\_size: 100, avoid\_regress: false} \\
scann[b] & \footnotesize{num\_leaves: 200, num\_leaves\_to\_search: 10, use\_score\_ah: false, reorder\_size: null, avoid\_regress: false} \\
scann[c] & \footnotesize{num\_leaves: 500, num\_leaves\_to\_search: 10, use\_score\_ah: true, reorder\_size: 100, avoid\_regress: false} \\
scann[d] & \footnotesize{num\_leaves: 500, num\_leaves\_to\_search: 10, use\_score\_ah: false, reorder\_size: null, avoid\_regress: false} \\
scann[e] & \footnotesize{num\_leaves: 200, num\_leaves\_to\_search: 10, use\_score\_ah: true, reorder\_size: 100, avoid\_regress: true} \\
scann[f] & \footnotesize{num\_leaves: 200, num\_leaves\_to\_search: 10, use\_score\_ah: false, reorder\_size: null, avoid\_regress: true} \\
scann[g] & \footnotesize{num\_leaves: 500, num\_leaves\_to\_search: 10, use\_score\_ah: true, reorder\_size: 100, avoid\_regress: true} \\
scann[h] & \footnotesize{num\_leaves: 500, num\_leaves\_to\_search: 10, use\_score\_ah: false, reorder\_size: null, avoid\_regress: true} \\
nndescent[a] & avoid\_regress: false \\
nndescent[b] & avoid\_regress: true \\
sheesh[a] & \footnotesize{use\_rebuilds: true, num\_prev\_assignments: 10, enable\_seeds: true, enable\_bulk: true, enable\_min\_iter: true} \\
sheesh[b] & \footnotesize{use\_rebuilds: false, num\_prev\_assignments: 10, enable\_seeds: true, enable\_bulk: true, enable\_min\_iter: true} \\
sheesh[c] & \footnotesize{use\_rebuilds: false, num\_prev\_assignments: 10, enable\_seeds: true, enable\_bulk: false, enable\_min\_iter: true} \\
sheesh[d] & \footnotesize{use\_rebuilds: false, num\_prev\_assignments: 10, enable\_seeds: true, enable\_bulk: false, enable\_min\_iter: false} \\
sheesh[e] & \footnotesize{use\_rebuilds: false, num\_prev\_assignments: 1, enable\_seeds: true, enable\_bulk: false, enable\_min\_iter: false} \\
sheesh[f] & \footnotesize{use\_rebuilds: false, num\_prev\_assignments: 1, enable\_seeds: false, enable\_bulk: false, enable\_min\_iter: false} \\
\end{tabular}
\caption{Legend for algorithms with multiple parameter variations.}
\label{tab:param-var}
\end{table}

We tested each algorithm on each dataset
(listed in \cref{sec:exp-setup})
with each variation of parameters,
using a $500$ second timeout
--- we record the score up to and including the first iteration
that exceeds the $500$ second threshold,
at which point we halt the algorithm.
Note that many such scores vastly exceed the $500$ second threshold
for cases in which the algorithm takes a long time to compute a single iteration.
Note also that the cuML implementation of Lloyd's algorithm
ran out of VRAM for several cases (from which it is omitted),
but generally exhibited similar performance to our PyTorch implementation in those where it did not.
We have plotted all of our data
in \cref{fig:final-results-1,fig:final-results-2,fig:final-results-3}.
We also performed
one additional limited test, to demonstrate the scalability of our algorithm,
plotted in \cref{fig:final-results-1b}.
To aid in the reading of our plots,
we have sorted all entries in each legend by their
best score.
This is true of \emph{every plot in the paper}.
In particular, with this information, one can see that
our algorithm (with all features enabled, sheesh [a] from the table)
achieved the best score in every single experiment we performed.
In addition, any plot entries marked with a ``*'' timed out
(took $>500$ seconds to finish their first iteration).

\begin{figure*}
    \centering
    \includegraphics[scale=0.50]{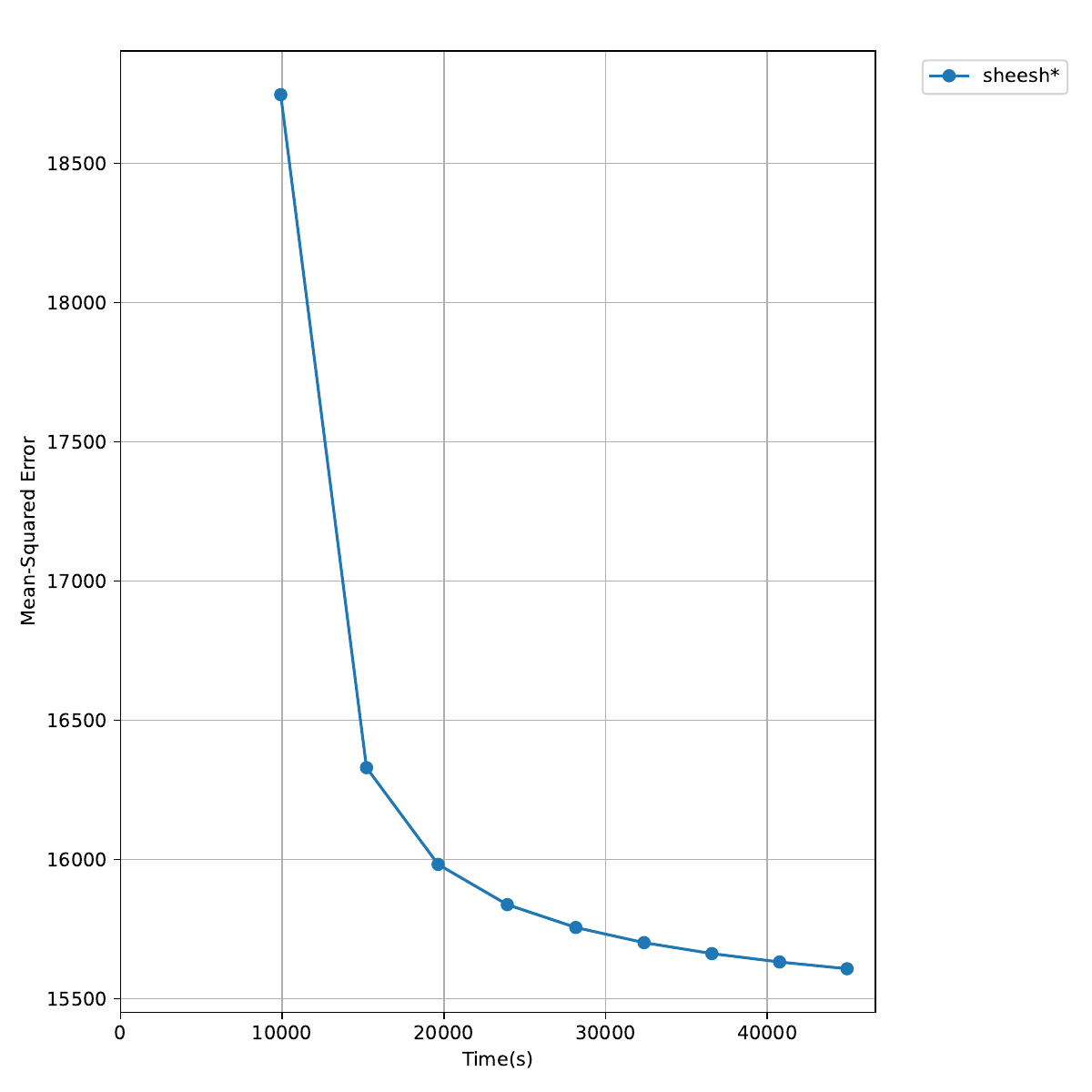}
    \caption{
    A plot of SHEESH running on SIFT1B with $k=1\,000\,000$
    for just over $12$ hours.
    Initialization is uniformly random.
    We estimate SciKit-Learn would take roughly 9.5 days to run a single iteration in this case.
    }
    \label{fig:final-results-1b}
\end{figure*}

\begin{figure*}
    \centering
    \begin{subfigure}{0.49\textwidth}
        \centering
        \includegraphics[scale=0.40]{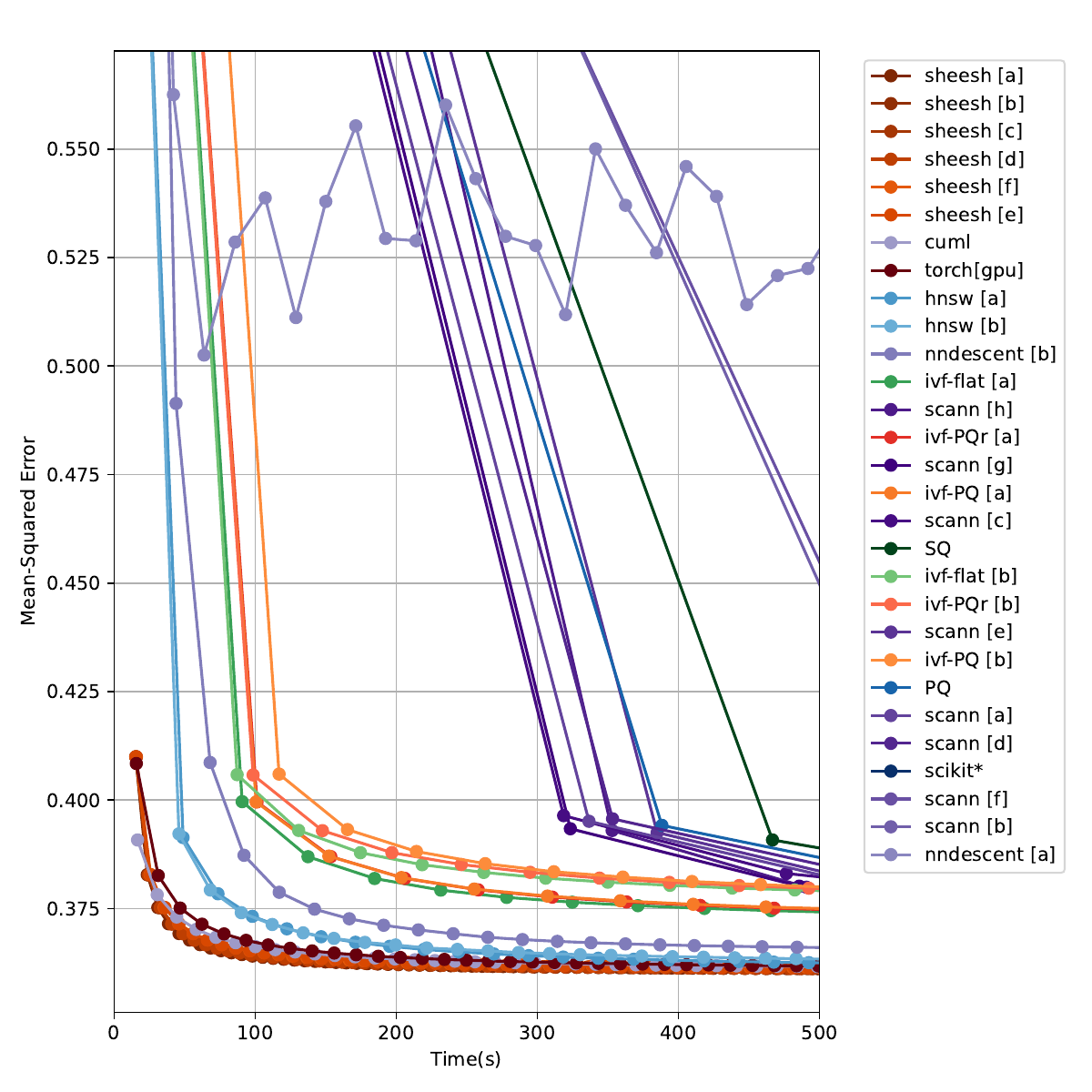}
        \caption{$k=10\,000$}
    \end{subfigure}
    \begin{subfigure}{0.49\textwidth}
        \centering
        \includegraphics[scale=0.40]{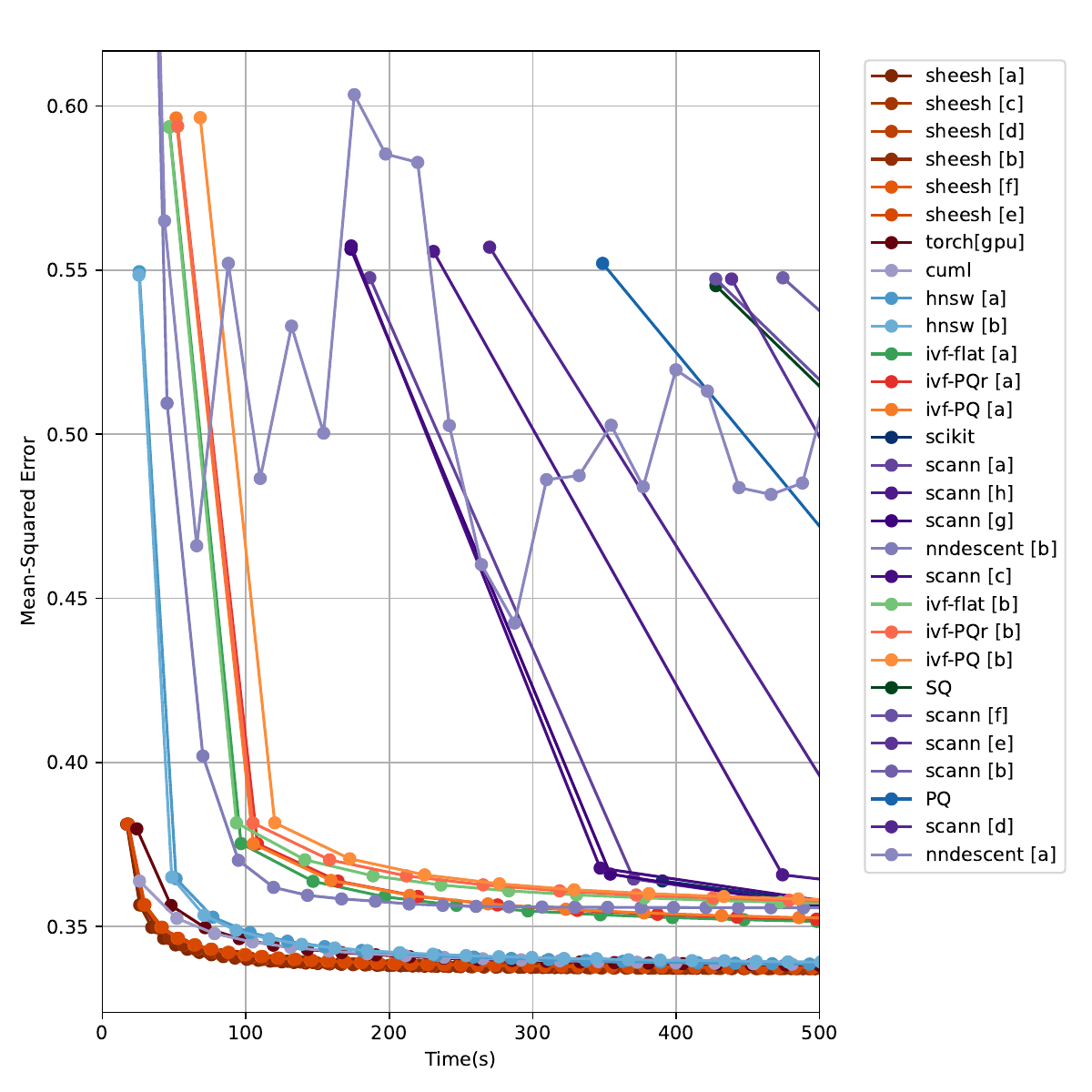}
        \caption{$k=20\,000$}
    \end{subfigure}
    \begin{subfigure}{0.49\textwidth}
        \centering
        \includegraphics[scale=0.40]{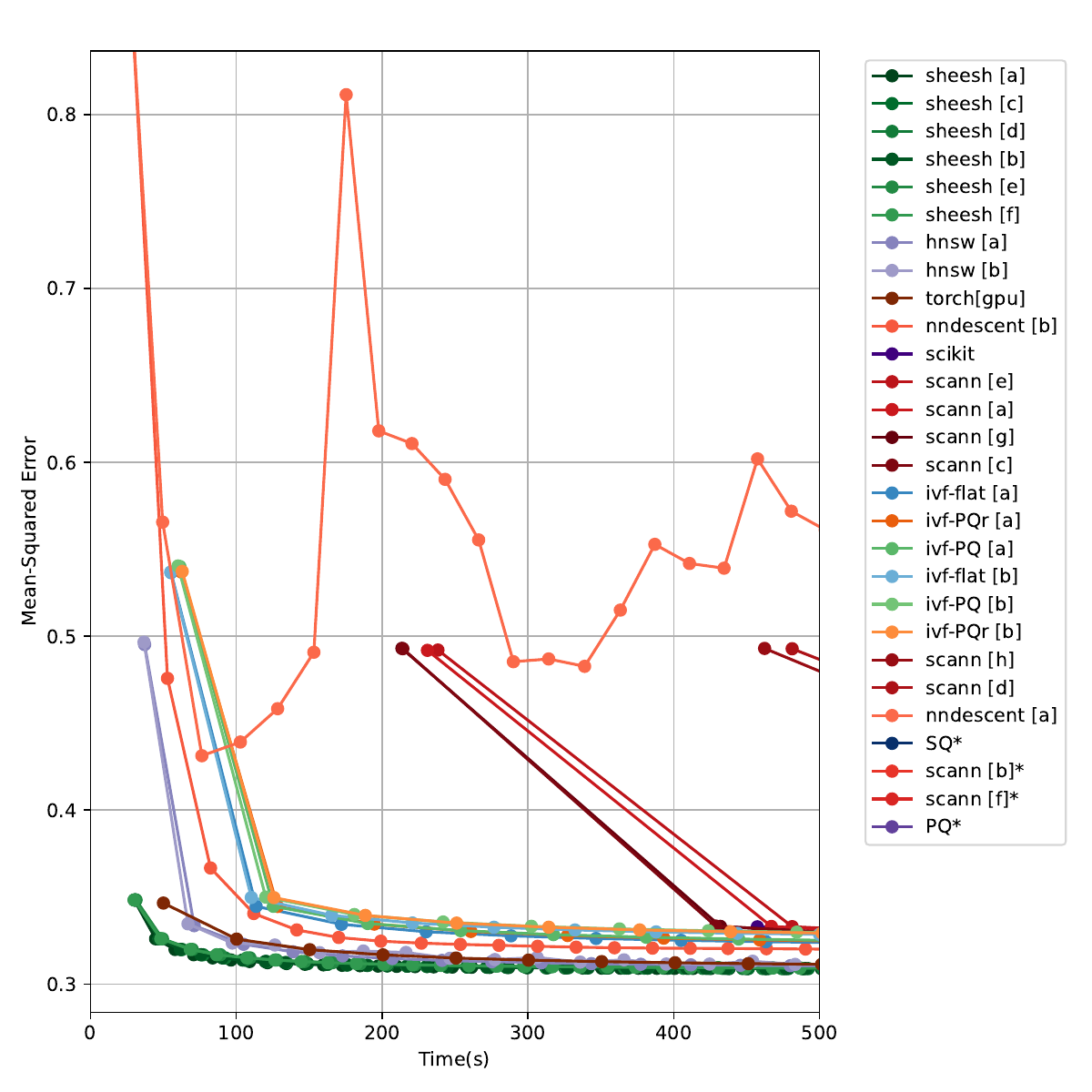}
        \caption{$k=50\,000$}
    \end{subfigure}
    \begin{subfigure}{0.49\textwidth}
        \centering
        \includegraphics[scale=0.40]{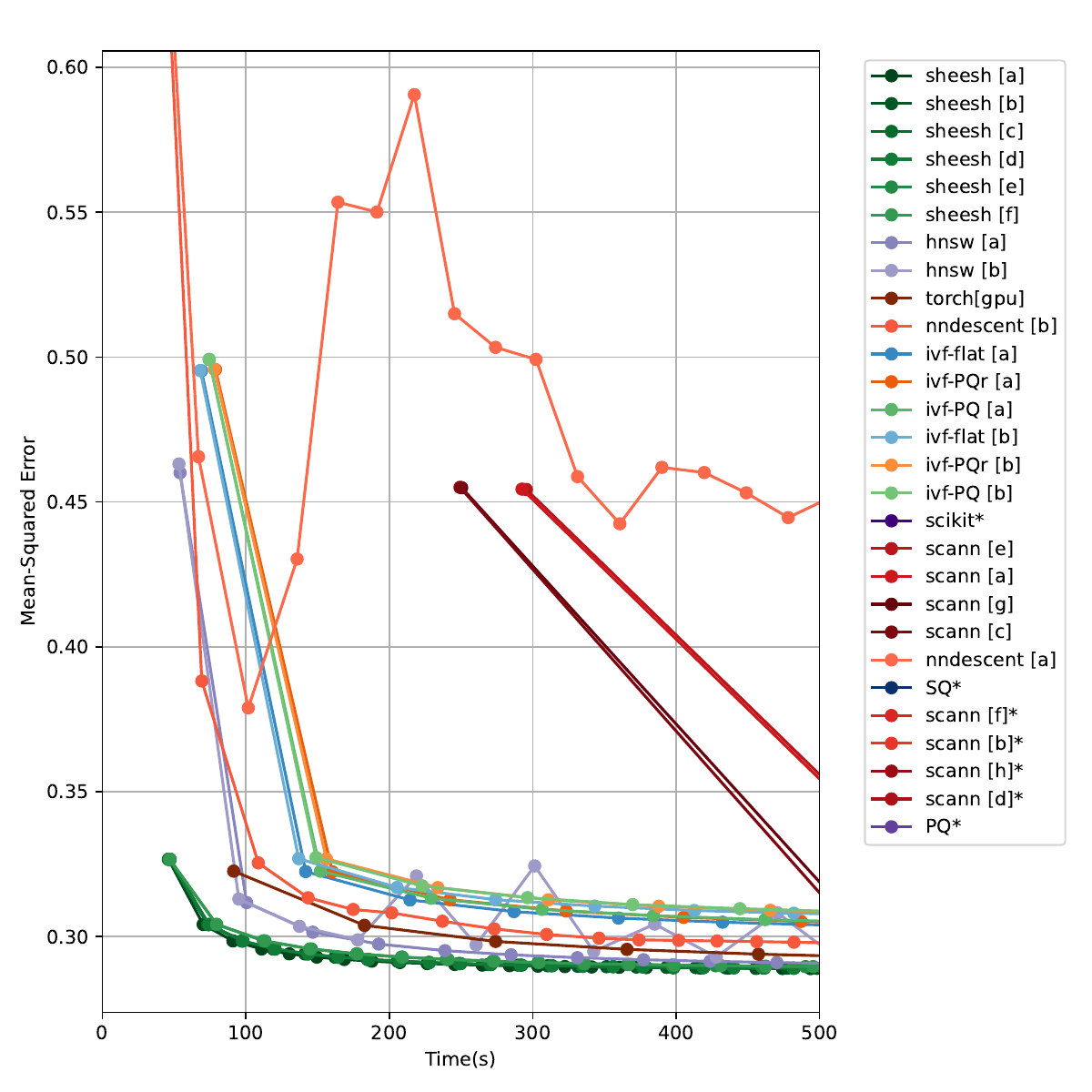}
        \caption{$k=100\,000$}
    \end{subfigure}
    \caption{Comparisons of all methods on the Text2Image10M dataset,
    for all tested values of $k$. Initialization is uniformly random.}
    \label{fig:final-results-1}
\end{figure*}

\begin{figure*}
    \centering
    \begin{subfigure}{0.49\textwidth}
        \centering
        \includegraphics[scale=0.40]{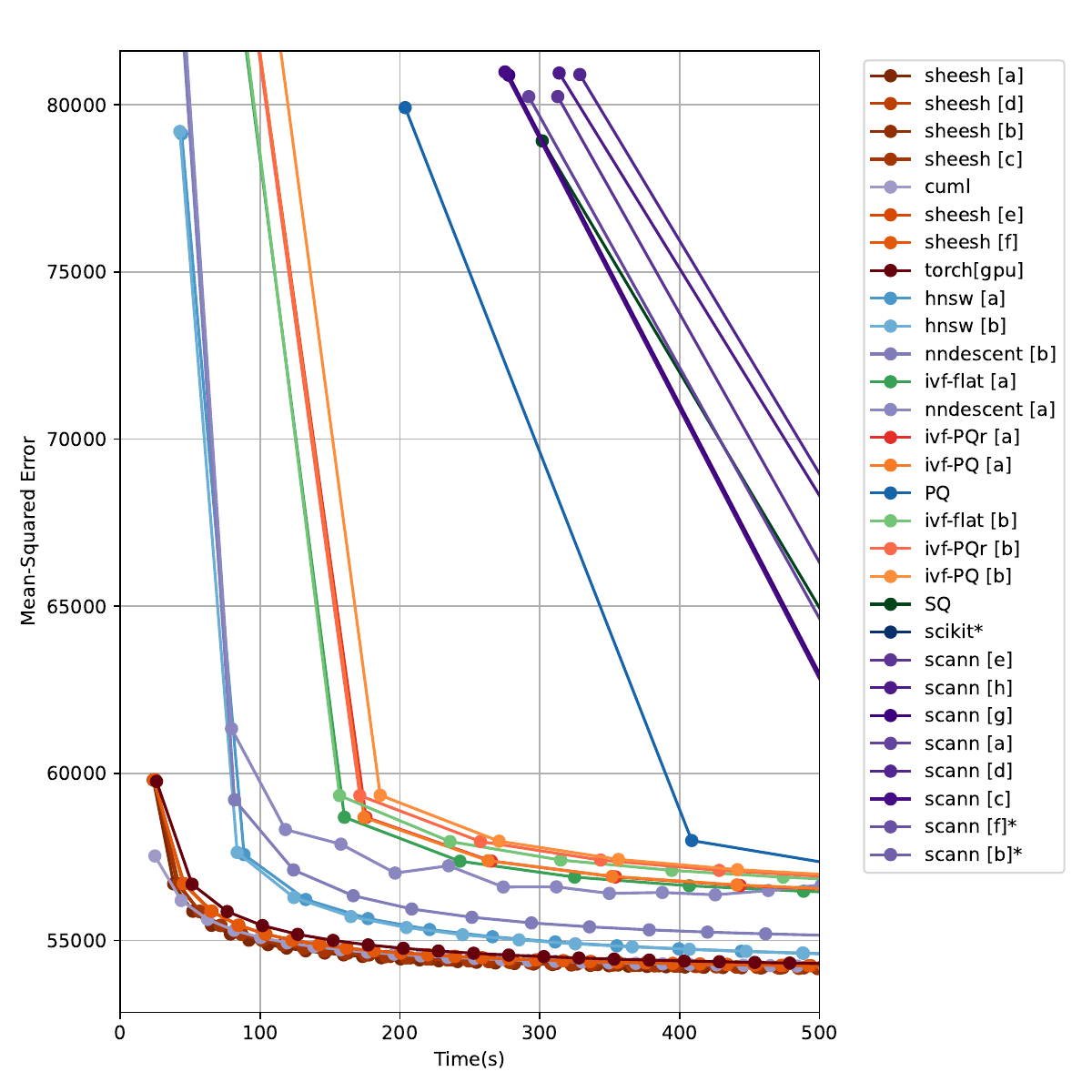}
        \caption{$k=10\,000$}
    \end{subfigure}
    \begin{subfigure}{0.49\textwidth}
        \centering
        \includegraphics[scale=0.40]{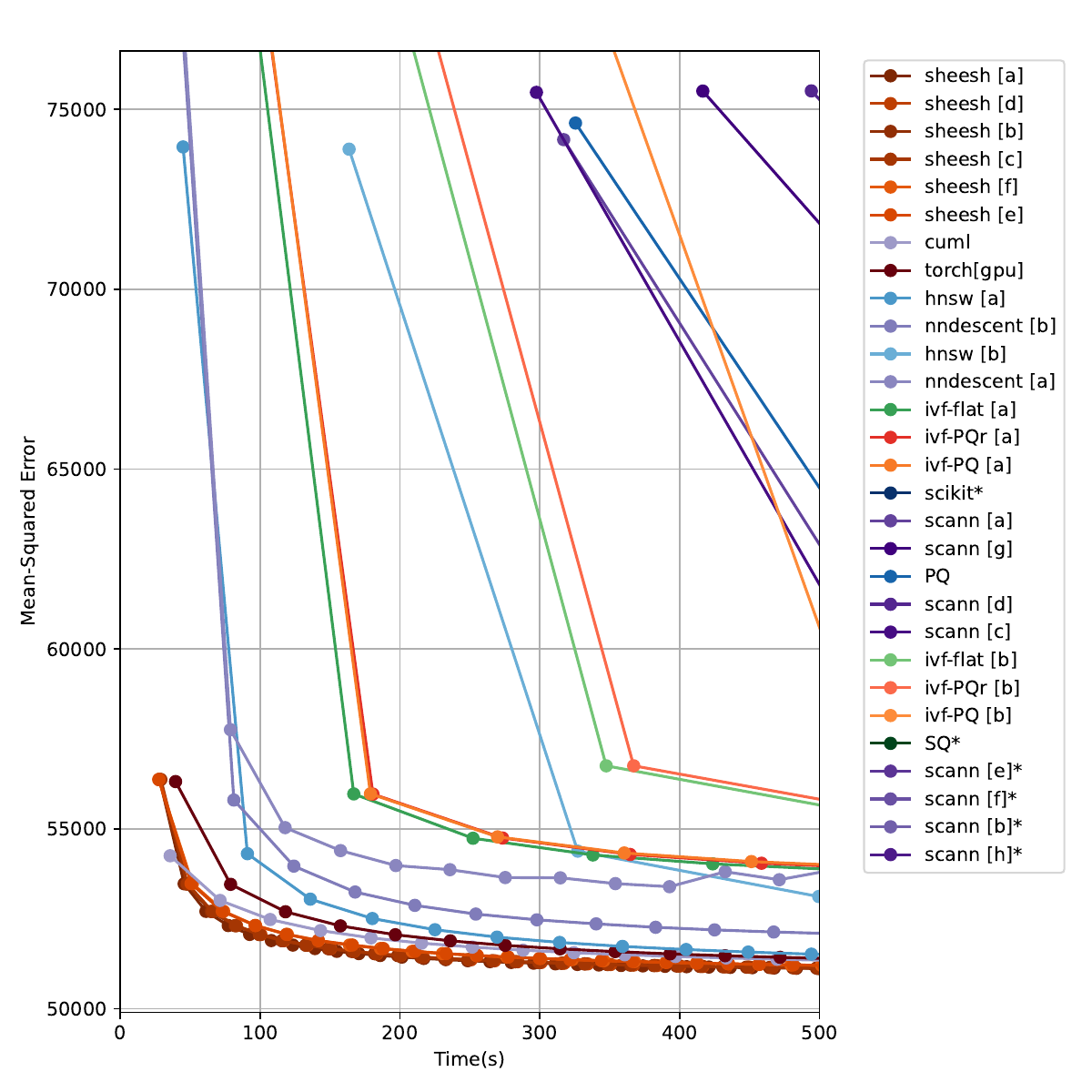}
        \caption{$k=20\,000$}
    \end{subfigure}
    \begin{subfigure}{0.49\textwidth}
        \centering
        \includegraphics[scale=0.40]{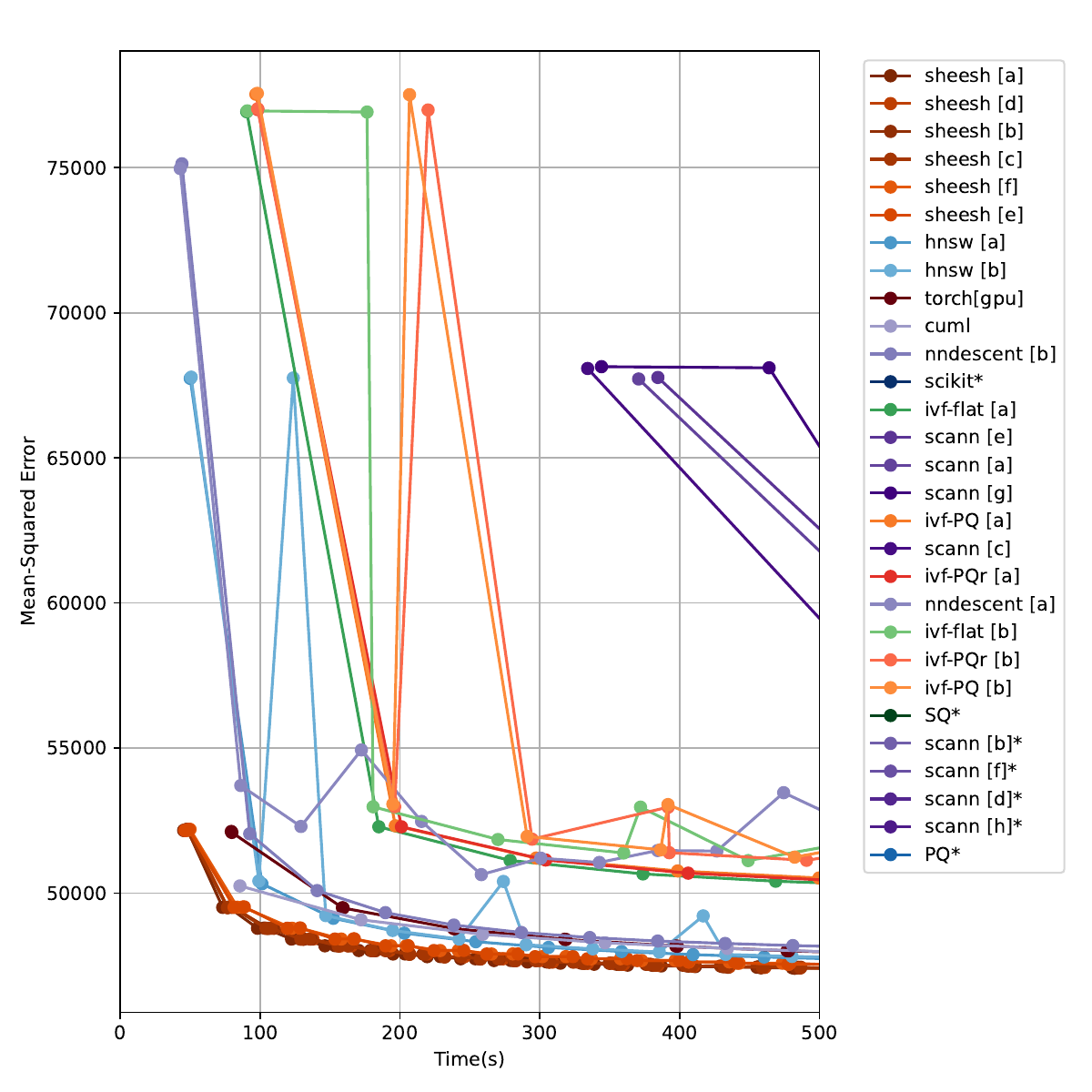}
        \caption{$k=50\,000$}
    \end{subfigure}
    \begin{subfigure}{0.49\textwidth}
        \centering
        \includegraphics[scale=0.40]{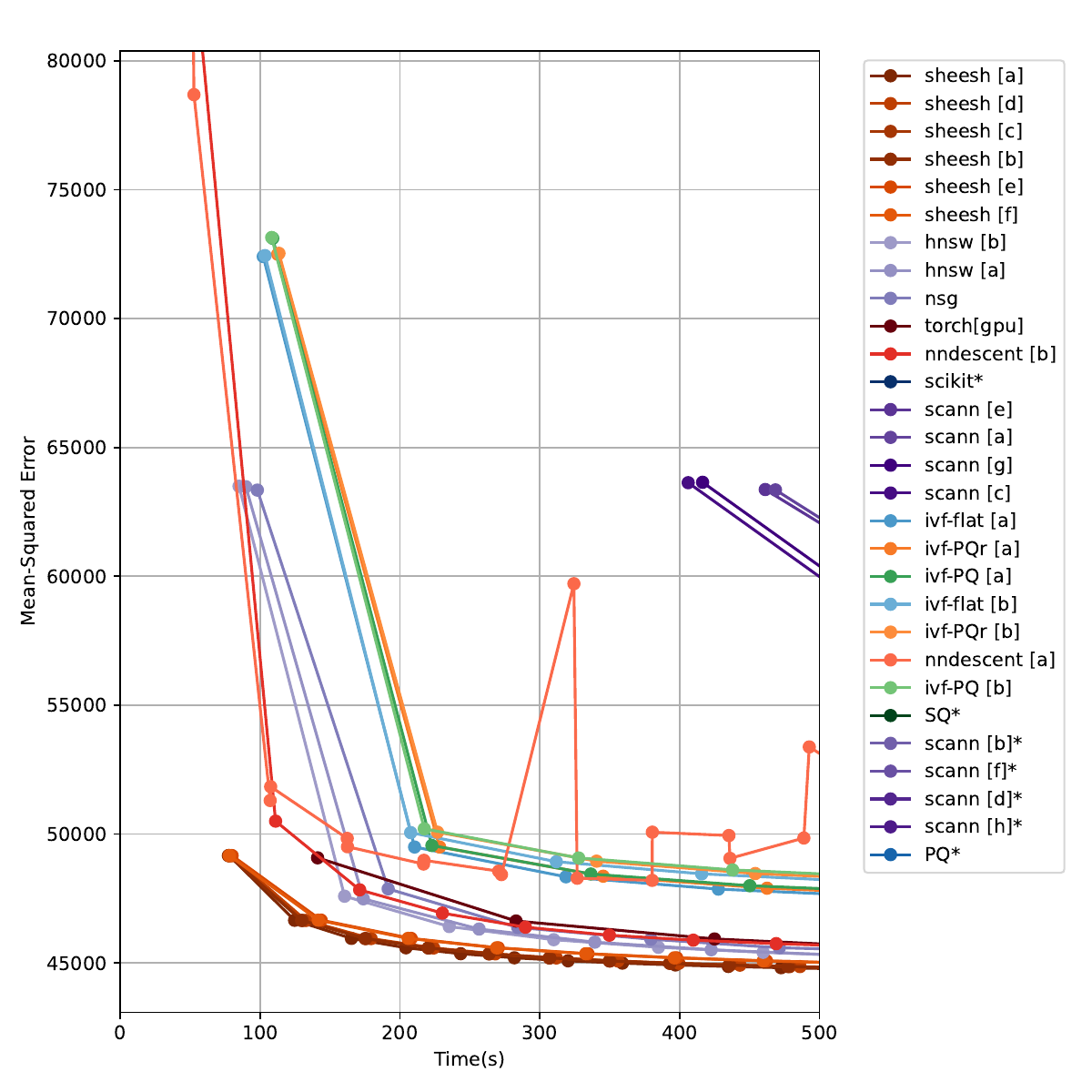}
        \caption{$k=100\,000$}
    \end{subfigure}
    \caption{Comparisons of all methods on the SIFT20M dataset,
    for all tested values of $k$. Initialization is uniformly random.}
    \label{fig:final-results-2}
\end{figure*}

\begin{figure*}
    \centering
    \begin{subfigure}{0.49\textwidth}
        \centering
        \includegraphics[scale=0.40]{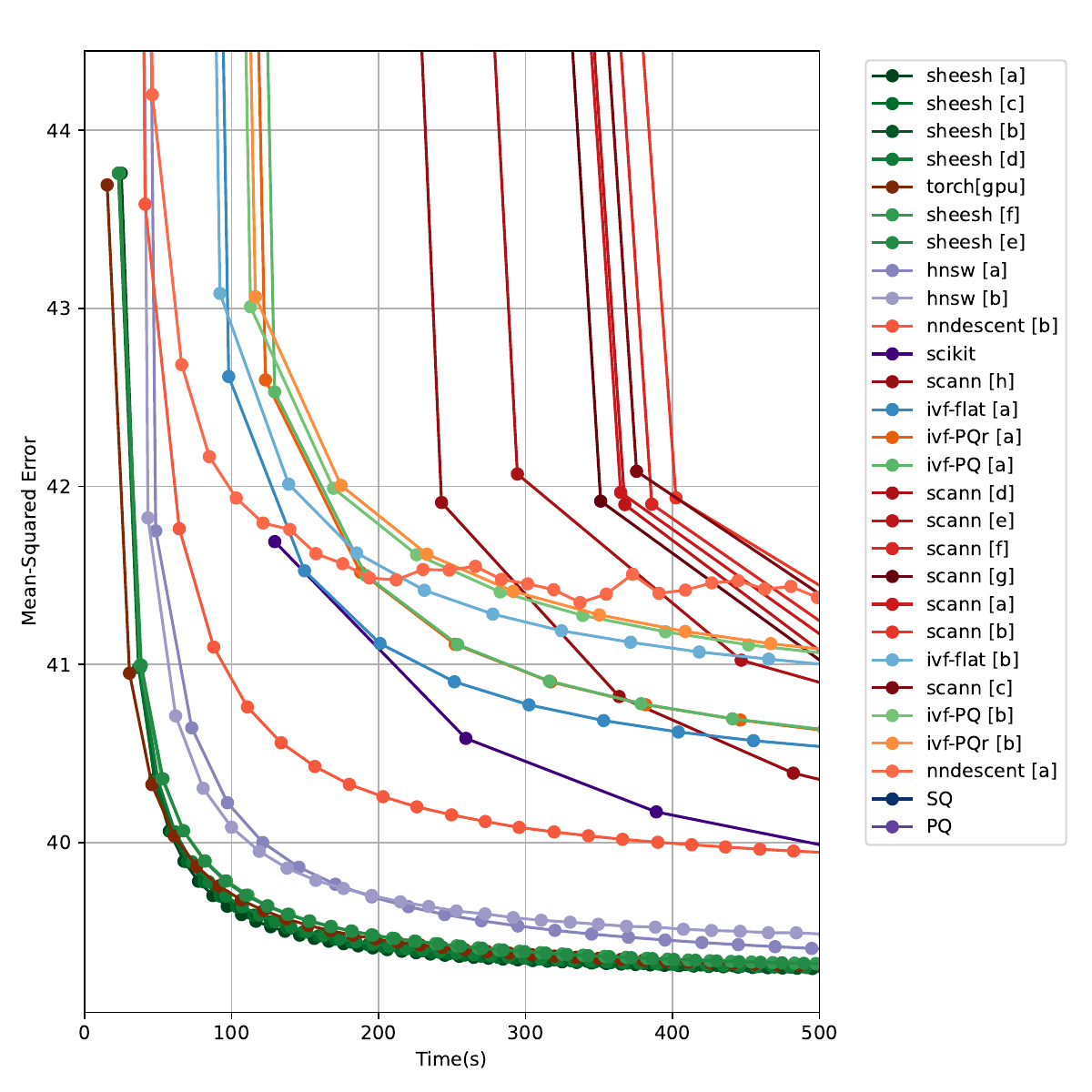}
        \caption{$k=10\,000$}
    \end{subfigure}
    \begin{subfigure}{0.49\textwidth}
        \centering
        \includegraphics[scale=0.40]{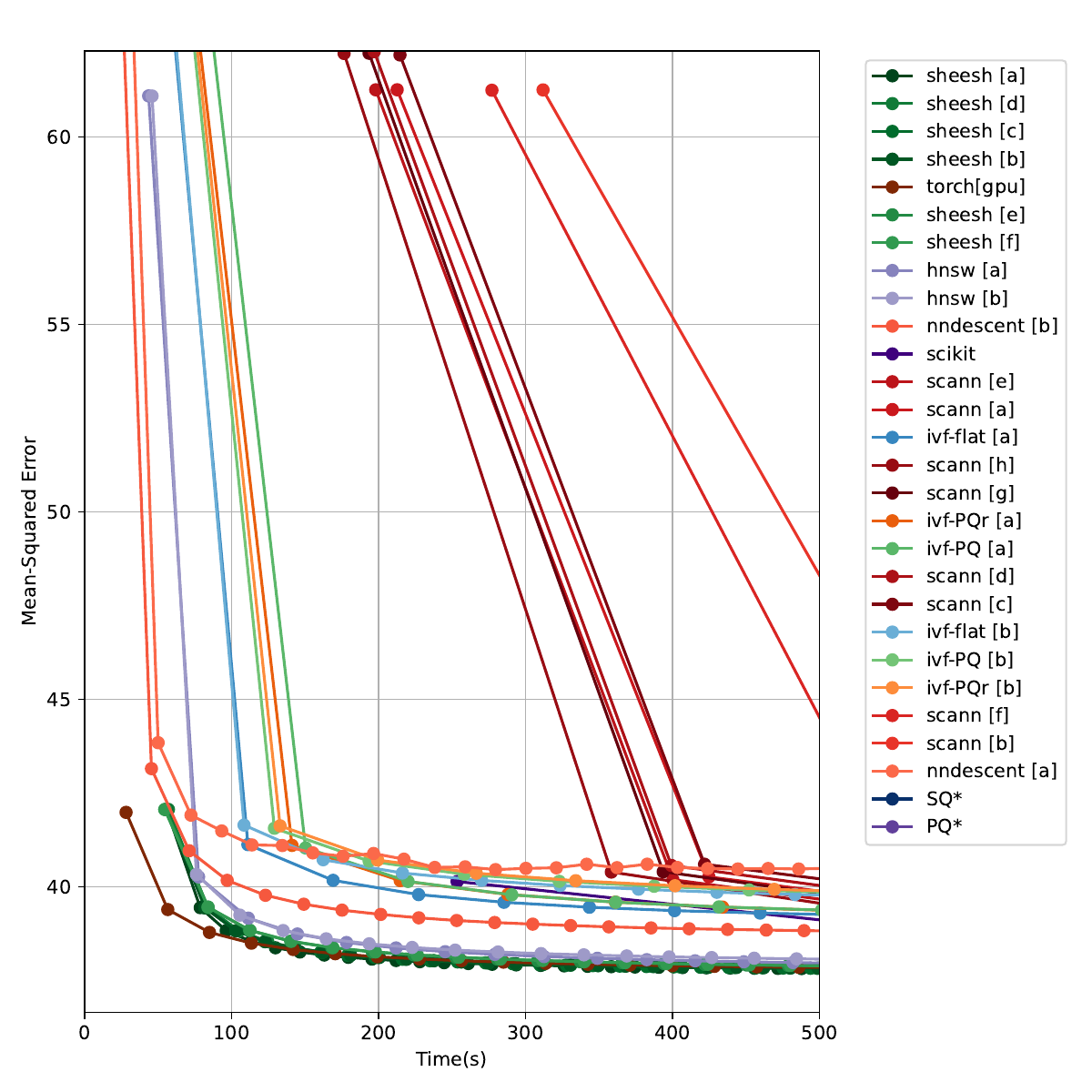}
        \caption{$k=20\,000$}
    \end{subfigure}
    \begin{subfigure}{0.49\textwidth}
        \centering
        \includegraphics[scale=0.40]{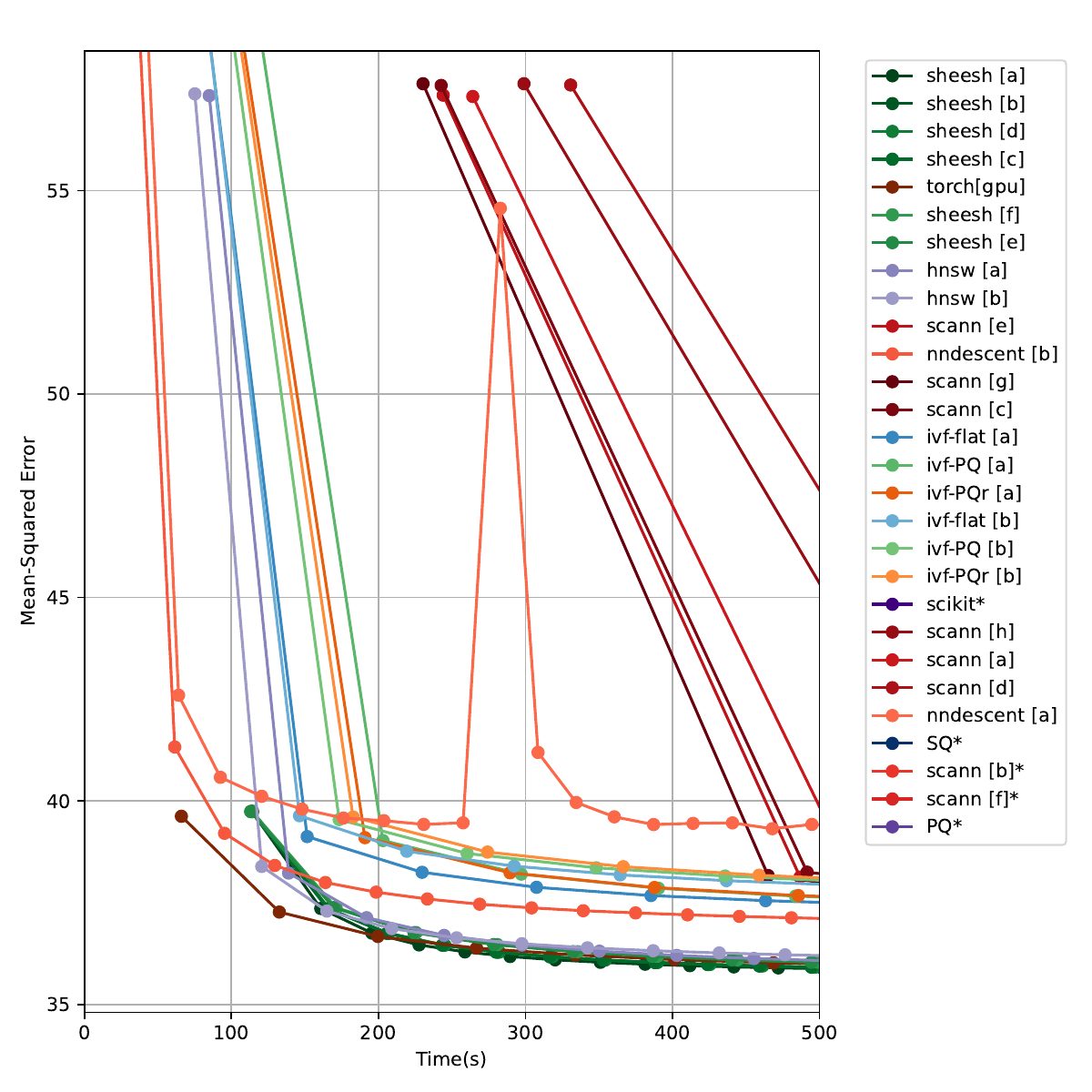}
        \caption{$k=50\,000$}
    \end{subfigure}
    \begin{subfigure}{0.49\textwidth}
        \centering
        \includegraphics[scale=0.40]{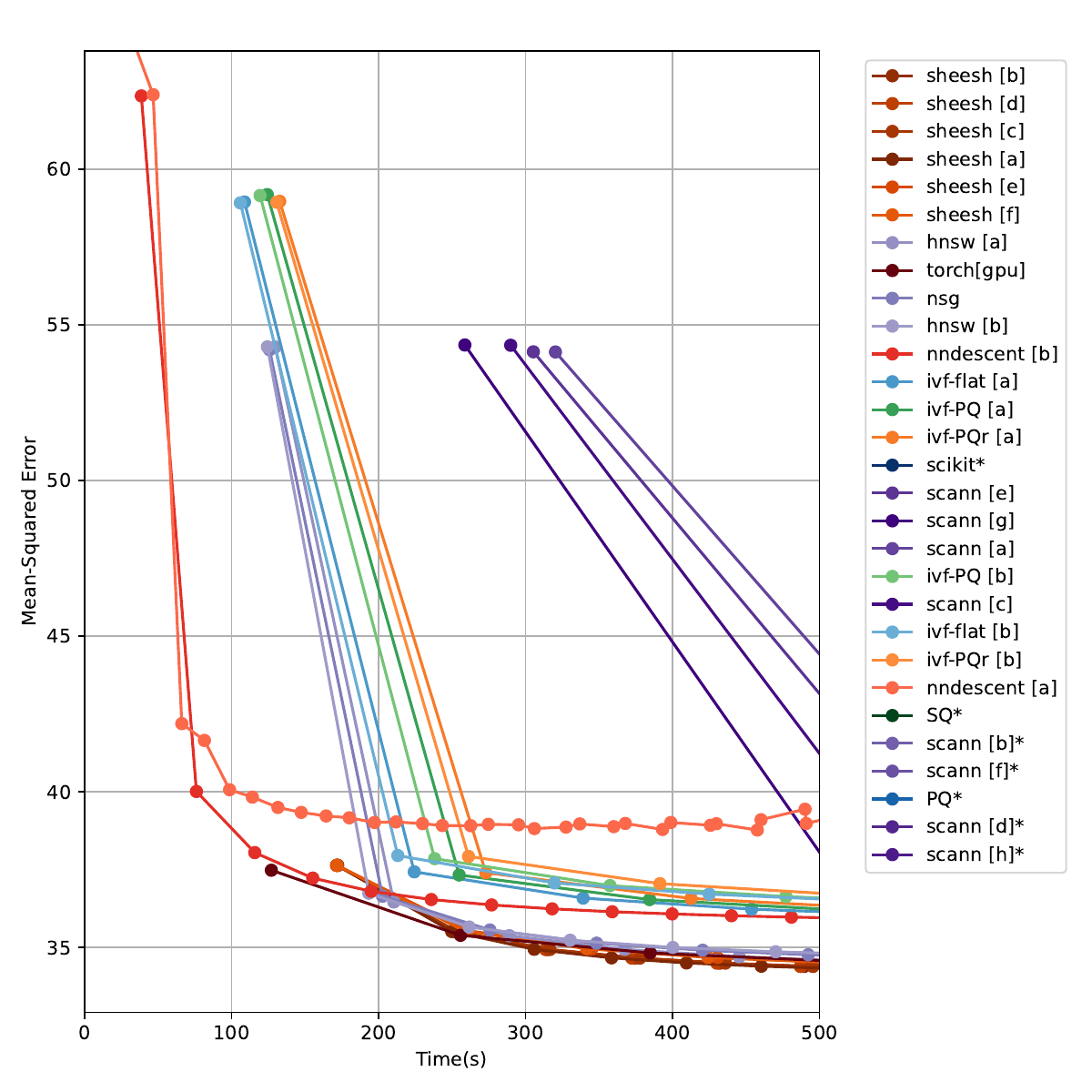}
        \caption{$k=100\,000$}
    \end{subfigure}
    \caption{Comparisons of all methods on the DPR5M dataset,
    for all tested values of $k$. Initialization is uniformly random.}
    \label{fig:final-results-3}
\end{figure*}

\newpage
\section{A Graph-Search Algorithm for SANNS with Provable Guarantees}
\label{sec:provable-sanns}

In this section, we observe that analysis for
an existing search-graph algorithm with provable guarantees for ANNS
(given by \citet{indyk2023worst})
naturally extends to become a seeded search-graph algorithm with provable guarantees for SANNS.
That is, we show that the algorithm
easily adapts to become a form of learning-augmented algorithm,
with robustness and consistency guarantees.

\citet{indyk2023worst} present a modified form
of the ``Vamana'' data structure given by \citet{jayaram2019diskann} for ANNS over a pointset $P\subset\Real^d$
with some provable guarantees.
For clarity, we will refer to the modified data structure as \defn{Vamana with slow preprocessing},
or simply \defn{VamanaSP}.
In particular, their provable guarantees are in terms of several parameters we must introduce:
\begin{itemize}
    \item The \defn{aspect ratio} $\Delta$ is a property of $P$.
    Specifically, it is the ratio $D_{\max}/D_{\min}$
    between the distance of the furthest pair
    $D_{\max}$ and the distance of the closest pair $D_{\min}$
    (where $D$ is the distance function of the metric space, usually Euclidean distance in the rest of our paper).
    \item The \defn{doubling dimension} $d'$ is also a property of $P$.
    For simplicity of presentation, we omit its format definition here,
    but it can be considered a measure of ``intrinsic dimensionality'' of the dataset
    in the same sense as discussed at the end of \cref{sec:ann_background}.
    \item The parameter $\alpha>1$ is a preprocessing-time parameter of both Vamana and VamanaSP.
    \item The parameter $\epsilon>0$ is a query-time parameter for tuning the approximation ratio of the nearest-neighbor returned by VamanaSP.
\end{itemize}

\citet{indyk2023worst} give a preprocessing algorithm for VamanaSP
running in $O(|P|^3)$ time.
This constructs a search-graph,
which they show has maximum degree $O((4\alpha)^{d'}\log\Delta)$.
Using \cref{alg:beam} for queries with $b=1$,
they show that only
$O\left(\log_\alpha\frac{\Delta}{(\alpha-1)\varepsilon}\right)$ node visits
are sufficient to find a $\left(\frac{\alpha+1}{\alpha-1}+\varepsilon\right)$-approximate
nearest-neighbor.
Note that each node visit requires
$O((4\alpha)^{d'}\log\Delta)$ distance computations.
Note that doubling dimension is NP-hard to compute~\cite{GottliebK13},
so it is unclear if it is expected to be a small quantity in a typical dataset.

We claim that,
if seeded with a ``learned''
element of $P$,
running \cref{alg:beam}
on their constructed graph
constitutes a form of learning-augmented algorithm.
In particular, it has the following two high-level properties,
in terms of the tradeoff between iteration count and approximation ratio:
\begin{itemize}
    \item \defn{Robustness}: It maintains worst-case guarantees.
    \item \defn{Consistency}: If the seed point already has a good approximation ratio,
    then the tradeoff between iteration count and approximation ratio \emph{improves}.
\end{itemize}
To prove this, we will simply leverage the techniques
of \citet{indyk2023worst}.
In particular,
they showed that
the aforementioned guarantees
for \cref{alg:beam} on VamanaSP
hold \emph{regardless of the initial starting vertex},
meaning we obtain robustness
for free from their analysis.

We now present a brief proof that (seeded) VamanaSP
also has a form of consistency,
which we will again prove by leveraging
the analysis of \citet{indyk2023worst}:
\begin{theorem}
Let $q$ be a query point
whose nearest-neighbor in $P$
is $a$.
Let $s\in P$ be a point
so that $1+\delta\geq\frac{D(s,q)}{D(a,q)}$,
for some value $\delta>0$.
Then \cref{alg:beam} starting at $s$
returns a
$\left(\frac{\alpha+1}{\alpha-1}+\varepsilon\right)$-approximate
nearest-neighbor
in
$\left\lceil\log_\alpha\frac{1+\delta}{\epsilon}\right\rceil$ node visits.
\end{theorem}

\begin{proof}
Let $p_i$ be the $i$th point visited
during the execution of \cref{alg:beam}.
Let $d_i$ be the distance $D(p_i,q)$.
As part of their proof, \citet{indyk2023worst}
show that $d_i\leq\frac{D(s,q)}{\alpha^i}+\frac{\alpha+1}{\alpha-1}D(a,q)$.
In particular, it also follows from their argument
that
the algorithm will only terminate without intervention if
it has found an
$\left(\frac{\alpha+1}{\alpha-1}\right)$-approximate
nearest-neighbor.
We leverage this in a straightforward manner:
If $i\geq\log_\alpha\frac{1+\delta}{\epsilon}$,
then $\varepsilon\geq\frac{1+\delta}{\alpha^i}\geq
\frac{D(s,q)}{D(a,q)\alpha^i}$.
Thus:
$$d_i\leq\frac{D(s,q)}{\alpha^i}+\frac{\alpha+1}{\alpha-1}D(a,q)
\leq\varepsilon D(a,q)+\frac{\alpha+1}{\alpha-1}D(a,q)=
\left(\frac{\alpha+1}{\alpha-1}+\varepsilon\right)D(a,q).$$
\end{proof}

Although \citet{indyk2023worst}
note that the aspect ratio of typical datasets
tends to be quite small,
this is not always necessarily the case.
Moreover, they note that
their algorithm essentially cannot have the worst-case
$\log\Delta$ factor in the number of node visits
replaced with a $\log|P|$ factor
(they give a more formal argument for this than we provide here).
Their proof of this relies on the fact that they
have no guarantees about the starting vertex.
Hence, it is reasonable to assume
that we may sometimes be able to leverage our consistency guarantee
to offer a slightly different algorithm with
an analogous improved ratio.
In fact, we can sometimes do even better than this,
by applying
methods that can obtain
$O(\log n)$-approximate nearest-neighbors
in well-known metric spaces.
In particular:
\begin{corollary}
Assume an oblivious adversary.
For data embedded in $\Real^d$,
under the Euclidean, inner-product, or cosine distances,
there exists an algorithm that, with $O(nd^2)$ preprocessing time,
can, in $O(d^2+\log n)$ time per query,
produce a seed point for VamanaSP
requiring
expected at most
$log_\alpha\frac{\min(O(\log n)}{\varepsilon}$ node visits
to produce a $\left(\frac{\alpha+1}{\alpha-1}+\varepsilon\right)$-approximate
nearest-neighbor.
\end{corollary}
In particular, such a seed point can be easily achieved
with a random projection of $P$,
and the proof follows.
Note that Euclidean/inner-product/cosine distance are all isometric
up to the inclusion of one extra dimension~\cite{bachrach2014speeding}.
Note also that
many other forms of (fixed approximation-ratio)
ANNS could be applied to further
accelerate this method.
However, we have highlighted the ability to use
random projection in particular
since it is somewhat related
to our use of random projection
in the main body of our work for BSANNS
(see \cref{subsec:seeded-search-graphs}).
For instance,
we believe it would be interesting to
determine if a ``bulk'' method
using \emph{only} random projection
(no grouping step)
could obtain interesting amortized guarantees.
That is, an approach where
a bulk query of $O(|P|)$ points
could be randomly projected
into $\Real$ for sorting purposes,
and then the results of each query
could be provided as seed points for the next.

Overall, in this section, we have given a nice extension of VamanaSP to SANNS.
Unfortunately, the preprocessing step of VamanaSP is too slow for our overall goal of $k$-means clustering with large $k$,
but this result is still of interest from the theoretical viewpoint of the SANNS problem.
We believe a promising avenue for
future theoretical work on SANNS would be to work with \emph{greedy trees}~\cite{ChubetPSS23}.
For fixed doubling dimension,
\citet{ChubetPSS23} show that greedy trees have better approximation guarantees than VamanaSP,
and
the work of \citet{Har-PeledM06} can be leveraged to show that
they can be computed with near-linear preprocessing time.

\newpage
\section{Implementation Details}
\label{sec:impl-deets}
In this section, we discuss some details of our specific implementation
that aided us in obtaining our final performance.

\subsection{HNSW Implementation}
Our base HNSW implementation has been carefully tuned.
The majority of its runtime when analyzed in a profiler (Linux perf~\cite{linux_perf} and KDBA Hotspot~\cite{KDAB_perf}) is taken up by distance computations.
In particular, most of these distance computations take place within
\cref{alg:beam}.
The distance computations are accelerated with careful manual \defn{prefetching} (preemptive insertions into cache) and SIMD-acceleration.
We implemented SIMD-accelerated distance queries with the Eigen {\tt C++} library ~\cite{eigen-10} for simplicity.
Manual prefetching is important since search-graph approaches to ANNS inherently have almost zero data-locality --- they involve comparing high-dimensional vectors according to an unpredictable graph traversal.
Consequently, compilers and hardware prefetchers are not well-equipped to predict
cache lines to prefetch.

In particular, the reference implementation for HNSW~\cite{malkov2018efficient} (`\texttt{hnswlib}'')
did not implement prefetching in the optimal manner.
We summarize their prefetching scheme for beam search
in \cref{alg:beam-hnswlib-prefetch} (slightly simplified for presentation).
We identified three areas for improvement in their implementation:
\begin{itemize}
    \item It only ever prefetches exactly one vector ahead (it is unparameterized).
    \item For each vector, it only fetches the first cache line containing the vector.
    For sufficiently high-dimensional data, it does not manually prefetch entire next vector (only first cache line). This may be mitigated by smart compilers or hardware prefetchers, but (as far as we know) such mitigations are not guaranteed.
    \item When the current vector has already been visited, it does not prefetch any part of the next vector, even if that vector is unvisited.
\end{itemize}
\begin{algorithm}[ht]
   \caption{beam search with hnswlib's prefetching scheme}
   \label{alg:beam-hnswlib-prefetch}
\begin{algorithmic}
   \STATE {\bfseries Input:} $P\subset\Real^d$, search-graph $G=(P,E)$, $p^*\in P$, $q\in\Real^d$, $b\in\mathbb{Z}_{\geq1}$
   \STATE Initialize sets $C,N=\{p^*\}$ (candidates, nearest).
   \STATE Mark $p^*$ as visited.
   \REPEAT
   \STATE Extract the element $c$ from $C$ nearest to $q$.
   \IF{$|N|=b$ and $d(c,q)>d(n,q)$ for all $n\in N$}
     \STATE \textbf{break}
   \ENDIF
   \STATE \textbf{Prefetch:} The first cache line of the data for the first neighbor of $c$.
   \FOR{each (outgoing) neighbor $v$ of $c$ in $G$}
     \IF{$v$ is not marked as visited}
     \STATE Mark $v$ as visited
     \STATE \textbf{Prefetch:} The first cache line of the data for the next neighbor of $c$ after $v$.
     \IF{$|N|<b$ or $d(v,q)<d(n,q)$ for some $n\in N$}
       \STATE Add $v$ to $C$ and $N$
       \STATE If $|N|>b$ or $|C|>b$, remove the furthest element.
     \ENDIF
     \ENDIF
     \STATE Mark $v$ as visited.
   \ENDFOR
   \UNTIL{$C$ is empty}
   \STATE {\bfseries Output:} $N$, the $b$ points in $P$ closest to $q$.
\end{algorithmic}
\end{algorithm}

Our own implementation is not based on \texttt{hnswlib}, and we use a more careful (and simpler) prefetching scheme, outlined in \cref{alg:beam-mop-prefetch} (again, slightly simplified for presentation).
The main ideas are as follows:
\begin{itemize}
    \item We make a list containing the indices of all unvisited neighbor \emph{before} iterating through the list.
    \item We parameterize the distance ahead that vectors are prefetched.
\end{itemize}
\begin{algorithm}[ht]
   \caption{beam search with our prefetching scheme}
   \label{alg:beam-mop-prefetch}
\begin{algorithmic}
   \STATE {\bfseries Input:} $P\subset\Real^d$, search-graph $G=(P,E)$, $p^*\in P$, $q\in\Real^d$, $b\in\mathbb{Z}_{\geq1}$
   \STATE Initialize sets $C,N=\{p^*\}$ (candidates, nearest), and an empty list $L$ (neighbor list).
   \STATE Mark $p^*$ as visited.
   \REPEAT
   \STATE Extract the element $c$ from $C$ nearest to $q$.
   \IF{$|N|=b$ and $d(c,q)>d(n,q)$ for all $n\in N$ and a sufficient number of iterations have occurred}
     \STATE \textbf{break}
   \ENDIF
   \FOR{each (outgoing) neighbor $v$ of $c$ in $G$}
       \IF{$v$ is not marked as visited}
           \STATE Add $v$ to $L$ and mark $v$ as visited.
       \ENDIF
   \ENDFOR
   \STATE \textbf{Prefetch:} All cache lines for the first $4$ elements in $L$.
   \FOR{each $v$ in $L$}
        \STATE \textbf{Prefetch:} All cache lines for the next element of $L$ that not yet prefetched.
        \STATE Compute $d(v,q)$.
        \IF{$|N|<b$ or $d(v,q)<d(n,q)$ for some $n\in N$}
            \STATE Add $v$ to $C$ and $N$.
            \STATE If $|N|>b$, remove the furthest element.
        \ENDIF
   \ENDFOR
   \STATE Clear $L$.
   \UNTIL{$C$ is empty}
   \STATE {\bfseries Output:} $N$, the $b$ points in $P$ closest to $q$.
\end{algorithmic}
\end{algorithm}

One of the NeurIPS 2023 Big-ANN competition~\cite{bigann2023} winners, PyANNS~\cite{pyanns},
also took a similar approach to prefetching.
In particular, they also automatically tuned the added parameter.
In contrast, we simply use a sane default of prefetching $4$ vectors ahead; a number of values seemed to exhibit essentially the same performance.

We have not carefully examined the prefetching schemes
of other ANNS search-graph implementations.

\subsection{Data Streaming}
Since we are dealing with datasets too large to fit in RAM,
we require some form of multi-threaded data streaming system.
We adopted a simple and straightforward approach
leveraging the C++ template system to create an abstract container we simply call a ``bucket''
implementing some kind of data streaming routine using callbacks.
We created several implementations of buckets,
including one wrapping NumPy containers \cite{harris2020numpy},
which is what we used for all of our benchmarking.
One could create similar bucket implementations
for a database or similarly bulky tool,
although we have opted to implement only
simple variations,
since
our methods are all quite compute-limited,
and the IO patterns are extremely simple
(we simply stream through the dataset in order during each iteration).

\end{document}